\documentclass[10pt,times]{elsarticle}

\usepackage[utf8]{inputenc} %
\usepackage[T1]{fontenc}    %

\usepackage{amsthm}

\usepackage{hyperref}       %

\usepackage{url}            %
\usepackage{booktabs}       %
\usepackage{amsfonts}       %
\usepackage{nicefrac}       %
\usepackage{microtype}      %
\usepackage[dvipsnames]{xcolor}         %
\usepackage{amsmath}
\usepackage{amssymb}
\usepackage{array}
\usepackage{amsthm}
\usepackage{svg}
\usepackage{bbm}
\usepackage{bm}
\usepackage{algorithm}
\usepackage{algpseudocode}
\usepackage{cleveref}
\usepackage[normalem]{ulem}
\usepackage{comment}
\usepackage{multirow}
\usepackage{placeins}

\usepackage{nkj}
\usepackage{enumitem}

\usepackage{caption}
\usepackage{subcaption}

\newcommand{\msgb}[1]{{\overleftarrow{\bfmu}^{(#1)}}}

\newcommand{\msgbel}[1]{\overleftarrow{\mu}^{(#1)}}

\newcommand{\Rel}{R}
\newcommand{\prelvec}[1]{ {\widetilde{\bfr}^{(#1)}}}

\newcommand{\numneu}[1]{{N^{(#1)}}}

\newcommand{\ol}{L}

\newcommand{\il}{0}

\newcommand{\subg}[1]{ {\mcS}^{(#1)}}

\newcommand{\fconst}[1]{{C^{(#1)}}}

\ifdefined\QED
\else
\newcommand{\QED}{\hfill \ensuremath{\Box}}
\fi

\makeatletter
\renewcommand\paragraph{\@startsection{paragraph}{4}{\z@}%
  {3.25ex \@plus1ex \@minus.2ex}%
  {-1em}%
  {\normalfont\normalsize\bfseries\unskip}}
\makeatother

\newtheorem{theorem}{Theorem}

\newtheorem{definition}{Definition}

\usepackage{accents}

\usepackage[most]{tcolorbox}

\newtcolorbox{commentbox}[2][]{%
  breakable,
  colback=yellow!2,        %
  colframe=gray!180!black,    %
  coltitle=black,           %
  fonttitle=\bfseries,      %
  title={#2},               %
  sharp corners,            %
  boxrule=1pt,              %
  left=6pt, right=6pt, top=4pt, bottom=4pt, %
  #1                        %
}

\usepackage{colortbl}

\begin{document}

\begin{frontmatter}

  \title{Normalized Relevance Measure as a Unifying Framework to Explain Neural Network Latent Structures}

 \author[1,2]{Ping Xiong}
\author[7,8,9]{Thomas Schnake}
\author[1,6]{Grégoire Montavon}
\author[1,2,4,5]{Klaus-Robert M\"uller}
\author[1,2,3]{Shinichi Nakajima\corref{cor}}
\ead{nakajima@tu-berlin.de}

\cortext[cor]{Corresponding author
}

\address[1]{Berlin Institute for the Foundations of Learning and Data -- BIFOLD, 10623 Berlin, Germany}
\address[2]{Machine Learning Group, Technical University of Berlin, 
Berlin, Germany}
\address[3]{RIKEN AIP, 
Tokyo, Japan}
\address[4]{Department of Artificial Intelligence, Korea University, Seoul,
Korea}
\address[5]{Max Planck Institute for Informatics, 
Saarbr\"ucken, Germany}
\address[6]{Charit\'e – Universit\"atsmedizin Berlin, %
Berlin, Germany}
\address[7]{Department of Chemistry, Chemical Physics Theory Group, University of Toronto, Toronto, 
Canada}
\address[8]{Vector Institute for Artificial Intelligence, Toronto, 
Canada}
\address[9]{Acceleration Consortium, University of Toronto, Toronto, 
Canada}

\begin{abstract}
To understand how a neural network (NN) functions and makes predictions, it has become increasingly clear that analyzing only the input domain is insufficient---one must also examine its internal inference mechanisms to capture the complete picture.   To explain the internal inference mechanisms of such models, it is essential to analyze the importance of latent representations for a given task.
In this paper, we propose the \emph{normalized relevance measure} (NRM) framework---a novel general explanation procedure that attributes relevance to \emph{arbitrary sets of neurons across layers of arbitrary architectures}.  In the NRM framework, relevance of selected neurons is explicitly defined as a normalized signed measure, constructed using simple operations---marginalization and conditioning based on additive and multiplicative laws---in analogy to the probability measures.  
The normalization property further guarantees comparability across layers.  The NRM framework subsumes existing propagation-based explanation algorithms by explicitly identifying the underlying quantity being computed.
We demonstrate the utility of the framework in computer vision applications, where joint relevance analysis across multiple layers reveals key information flows in VGG16 networks. Overall, the NRM framework provides a general, mathematically grounded approach to understanding how modern NNs propagate information, offering a versatile and broadly applicable foundation for explainable artificial intelligence.
\end{abstract}

\end{frontmatter}

\section{Introduction}
\label{sec:intro}

Deep neural networks (DNNs) are becoming increasingly popular due to their enormous learning capacity and outstanding performance in different downstream tasks.
However, these models are notoriously intransparent---it is often unclear why a DNN, even with a relatively simple architecture, makes particular predictions.
This lack of interpretability poses significant challenges, especially in high-stakes domains. To address this issue, the field of explainable artificial intelligence (XAI) has emerged  \citep{Gunning2019, DBLP:journals/pieee/SamekMLAM21, DBLP:conf/icml/HolzingerSMBS20, DBLP:series/lncs/11700}, aiming to enhance the transparency and interpretability of machine learning models.

A multitude of techniques have been proposed to explain the predictions of DNN and other ML models, often in terms of input features \cite{ancona2018towards,bach2015pixel,lund2017unified,lime} each of which having their distinct accuracy, applicability, and scalability profiles. Mathematical concepts, such as Taylor decomposition \citep{DBLP:journals/pr/MontavonLBSM17, schnake2020higher} and Shapley values \citep{Strumbelj2010,lund2017unified}, have been employed to establish theoretical foundations for XAI methods, and have been used to revisit explanation methods and extend them to new tasks or architectures \citep{DBLP:journals/pr/MontavonLBSM17, schnake2020higher, achtibat_attnlrp_2024,Bley2025}. %
\begin{figure*}[t]
    \begin{center}
        \includegraphics[width=\textwidth]{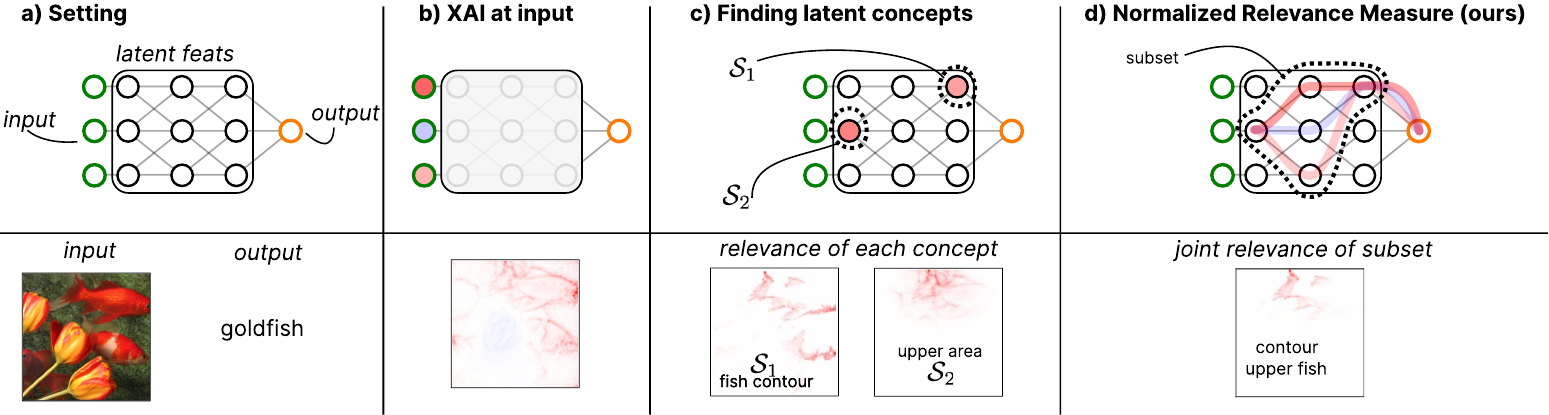}
    \end{center}
    \vspace{-5mm}
    \caption{
        Visualization of existing XAI methods and our NRM framework.
For given input and output of neural networks (a), conventional XAI methods typically attribute relevance feature-wise and only at the input layer (b). Some concept-based XAI methods attribute neurons in intermediate layers (c). Our proposed Normalized Relevance Measure (NRM) framework (d)
         offers an explicit and intuitive procedure for defining the normalized relevance of \emph{any set of neurons}, together with a systematic derivation of the corresponding LRP algorithms. Consequently, it enables us to 
    explain the flow of information through 
    neuron sets across layers, yielding more faithful, intuitive, and fine-grained explanations.
    }
    \label{fig:fig1_expl_concept}
\end{figure*}
A recent trend has been to explain prediction not only in terms of the input features, but also incorporating more abstract concepts found in the latent representation of the model (see, e.g., \cite{ConveptRelevancePropagation, DBLP:journals/tnn/KauffmannERMSM24}). To explain such models, 
 it is desirable to analyze the importance of latent representations and their interactions in order to elucidate internal inference mechanisms. 
Although several XAI approaches \citep{schnake2020higher, Chormai2024} have been proposed to analyze such internal mechanisms, no foundational framework yet exists that can quantify them in a flexible and consistent manner. 

In this paper, we propose a \emph{normalized relevance measure} (NRM) framework for attributing sets of neurons across different layers.  In this framework, the attribution quantity is defined as a signed measure, thereby satisfying basic rules analogous to the additive and multiplicative laws for probability measures.  This formulation allows us to easily define the relevance of any set of neurons through marginalization and conditioning operations.
Furthermore, the NRM framework can be applied to any type of network architecture by virtually transforming it into a restricted form of feed forward NN (FFNN).

The relevance of any set of neurons can be computed via layer-wise relevance propagation (LRP), which is systematically derived as sum-product message passing---a standard class of algorithms for probabilistic inference \citep{bishop2006patternchp8}. 
Extending the work in \citet{DBLP:conf/icml/XiongSMMN22}, we formalize the general relation between sets of neurons to be explained and the corresponding LRP algorithms.

The NRM framework thus provides a general procedure for deriving LRP algorithms that compute the normalized relevance of any 
set of neurons in any type of network architecture.
Although LRP algorithms for attributing intermediate neurons have been proposed \citep{ConveptRelevancePropagation}, they are defined as algorithmic procedures, i.e., how to propagate relevance, how to mask neurons, and when to stop the propagation. Consequently, it is not straightforward to define or search information flows that capture higher-order interactions \citep{schnake2020higher, DBLP:journals/pami/EberleBKMVM22}.
Furthermore, when neurons in different layers are attributed, there is no guarantee that the attributed quantities are comparable across layers.
In contrast, the attribution quantities in our NRM framework are directly defined, and their comparability is guaranteed by the normalization property. In Figure \ref{fig:fig1_expl_concept} we give a visual summary of our methodology, showing attribution for various sets of neurons.

NRM is a general framework, including 
any propagation-based method with conservation property.  We show that, under mild conditions, existing LRP algorithms can be cast within the NRM framework, thereby explicitly identifying the quantities they compute as normalized relevance measures, if irrelevant constant scale factors are ignored.
Furthermore, we discuss the possibility of extending the NRM framework %
 to occlusion- and gradient-based methods.  

To demonstrate the utility of our NRM framework, we perform higher-order analyses of information flow in VGG16. Specifically, we compute joint relevance across neurons in multiple layers, and identify interacting channels or neuron clusters  that focus on characteristic features of the target image class---patterns that may be overlooked by single-layer marginal analyses. Our quantitative evaluation provides strong evidence for the faithfulness of our approach relative to existing methods.
These results show that the NRM framework provides new tools for explaining NNs with fine-grained detail.

\medskip

The main contributions of this paper are summarized as follows:
\begin{itemize}
\item
We propose the NRM framework that allows us to systematically define normalized joint relevance of any set of neurons across layers, and to derive the corresponding LRP algorithm.

\item
We discuss how existing XAI methods can be cast within the NRM framework, providing a unified view of propagation-based methods.  We also discuss the possibility to extend the framework to gradient- and occlusion-based XAI methods.

\item
We apply the NRM framework to VGG16, and demonstrate that it can identify important higher-order information flow across layers. 

\end{itemize}

The paper is organized as follows.
After reviewing related work in \Cref{sec:RelatedWork}, 
we introduce the NRM framework in \Cref{sec:NormalizedRelevanceFramework}.
In \Cref{sec:ExperimentCV}, we demonstrate the utility of the NRM framework in qualitative and quantitative experiments, and \Cref{sec:Conclusion} concludes.

\section{Related Work}
\label{sec:RelatedWork}

Our work lies at the intersection of several subfields of Explainable AI, addressing the challenges of defining meaningful units of interpretability, analyzing interactions between or within these units, and producing robust, scalable explanations. Related work in each subfield is presented below.

\paragraph{Subgraph attribution} Several works focus on the problem of explanation specifically in graph neural networks and other structured ML models, identifying subgraphs as a natural interpretable unit for attributing predictions.  \citet{yuan2021explainability} proposes SubgraphX, a method based on Monte-Carlo tree search to find the most relevant subgraph in the input graph of a GNN.
\citet{ying2019gnnexplainer} proposes GNNExplainer, which synthesizes subgraphs that maximally align with the response of the GNN model. 
\citet{schnake2020higher} proposes GNN-LRP, an attribution to walks within a graph, which can be pooled into subgraphs. The approach is extended by \citet{DBLP:conf/icml/XiongSMMN22}, who proposed fast algorithms for subgraph attribution. In all of these works, the units of interpretability are subgraphs of the input graph. In contrast, our work focuses specifically on subgraphs of neurons within a feedforward architecture, with the goal of understanding the hierarchical decision-making processes of neural networks.

\paragraph{Explaining with interactions} Another category of works focuses on feature interactions, with the goal of understanding how they jointly contribute to the predicted output. In \citet{tsang2018detecting}, the authors propose to detect statistical interactions between input features by explaining a learned deep learning model. In \citet{Cui20, janizek2021integrated_hessian}, the authors recover pairwise interaction by analyzing higher-order derivatives of the network. \citet{Lundberg2020} propose an extension of Shapley values to recover similar pairwise feature interactions. \citet{DBLP:journals/pami/EberleBKMVM22} propose a robust approach based on LRP for explaining representation similarity in terms of pairs of input features. \citet{schnake2020higher} introduce \emph{walk} as a minimal explanation unit, capturing the interacting effect of nodes in an input graph or of neurons at each layer. \citet{DBLP:conf/icml/XiongSMMN22} point out that the relevance attributed to these walks has the same decomposability structure as the joint probability of a Markov chain, and derived a sum-product algorithm (cf.\ \citep{bishop2006patternchp8}) to calculate marginal relevance scores. \citet{SCHNAKE2025} proposed Symbolic XAI (SymbXAI), where walk relevance is used as a starting point to estimate the relevance of logical predicates over the input features. In our proposed NRM framework, a walk is the basic \emph{unit}, and the relevance of \emph{any set of neurons} in a deep NN is defined as a sum of walk relevances.

\paragraph{Layer-wise relevance propagation (LRP)}
\label{sec:RW.LRP}
Layer-wise Relevance Propagation (LRP) \citep{bach2015pixel} is a post-hoc, propagation-based XAI method for neural networks, where the implementation depends on the model architecture (i.e., it is model-specific). The method decomposes the model prediction into relevance scores of the input features, by a propagation scheme from the output layer to the input layer. 
Since LRP computation requires only a single backward propagation through the network, it is commonly used as a computationally efficient XAI tool. 
The explanation by LRP has been mathematically justified as the first-order Deep Taylor Decomposition (DTD) \citep{DBLP:journals/pr/MontavonLBSM17, DBLP:journals/pieee/SamekMLAM21}—namely, LRP approximates the non-linear 
forward process with layer-wise linear approximations.
LRP's advantageous computational properties enabled the analysis of datasets at scales and the systematic identification of flawed strategies in the model \citep{lapuschkin2019unmasking,kauffmann2025explainable}, so-called Clever-Hans effects.

\paragraph{Mechanistic interpretability} Another set of related works aim to understand the internal reasoning of neural network in more detail than through measuring the effect of input features. Methods such as  \cite{ConveptRelevancePropagation, Chormai2024, Zhou2018} propose to filter explanation through hidden neurons or latent subspaces, effectively capturing interaction between input features and latent. Several works focus on extending classical attribution methods to operate in latent representations \citep{dhamdhere2018how,gorbani2020neuron_shap,schnake2020higher}.  Another set of works \citep{olah2018the,elhage2021mathematical,Rai2024APR, NEURIPS2023_34e1dbe9} aims to uncover the causal, algorithm-like process that the model implements internally. NRM provides theoretical foundations for mechanistic interpretability, allowing us to define the relevance of any set of neurons across layers. Notably, the relevance quantity is \emph{explicitly defined} as a normalized signed measure, and thus is guaranteed to be \emph{comparable} across layers.  This comparability is essential when ranking the relevance of individual input features that interact with the network at different layers, as demonstrated in \Cref{sec:NormalizationProperty}.

\section{Normalized Relevance Measure Framework}
\label{sec:NormalizedRelevanceFramework}

\begin{figure}[t]
    \centering
    \includegraphics[width=\linewidth]{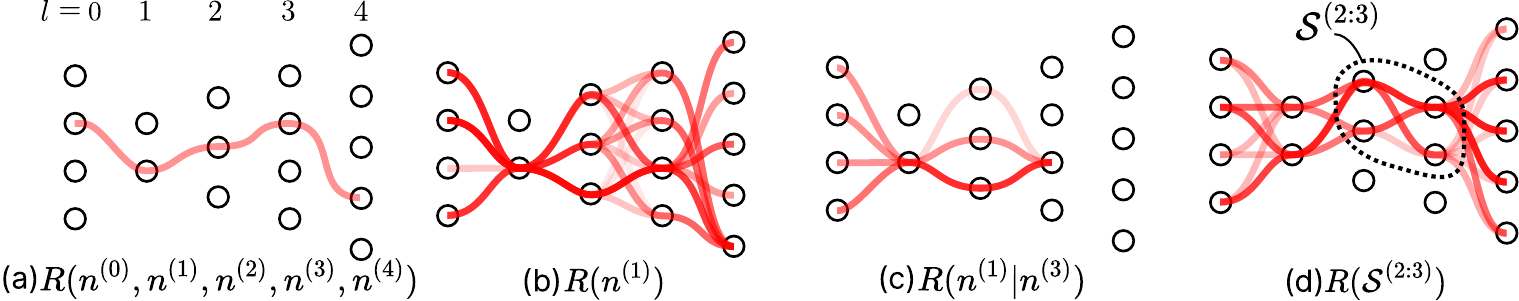}
    \caption{(a) Walk (full joint), (b) single neuron (marginal) , (c) conditional, and (d) substructure relevances. 
     The full joint relevance is the relevance of a single walk, and the marginal, conditional, and substructure relevances are the relevances of sets of walks and their ratios. 
    The variable $n^{(l)} \in \mathbb{N}^{(l)} = \{ 1, \ldots,  N^{(l)}\}$ specifies a single neuron at layer $l$.
    } 
    \label{fig:A.JointMarginalConditionalRelevance}
\end{figure}

In this section, we propose a novel Normalized Relevance Measure (NRM) framework.
We first provide motivation through a simple example in \Cref{sec:NRM.Motivation}, and then introduce the framework mathematically  
in Sections~\ref{sec:NormalizedSignedMeasure}--\ref{sec:MessagePassingAlgorithms}. 
We furthermore discuss its generality in \Cref{sec:GeneralityNRM} by showing how existing propagation-based algorithms can be cast within the NRM framework, by discussing an alternative relevance definition,
and by considering 
possible extensions to occlusion- and gradient-based XAI methods.
Lastly, we summarize the key features of NRM in \Cref{sec:KeyFeatures}.

\subsection{A simple example}
\label{sec:NRM.Motivation}

Consider a 4-layer MLP, shown in \Cref{fig:A.JointMarginalConditionalRelevance} (a).
For higher-order explanation, we aim to attribute relevance to an arbitrary set of neurons in a way that the relevance corresponds to the information flow passing through the selected neurons.

Let us first select one neuron per each layer, and denote its relevance as
\begin{align}    
   R(n^{(0)}, n^{(1)}, n^{(2)}, n^{(3)}, n^{(4)}).
   \label{eq:Motivation.Joint}
\end{align}
This notion of relevance, which measures the information flow passing through the neuron sequence $\bfn = (n^{(0)}, n^{(1)},\allowbreak  n^{(2)}, n^{(3)}, n^{(4)})$---a \emph{walk}---is defined in \citet{schnake2020higher}.

Then, let us remove neurons from the walk, and consider the relevance of a single neuron
$R(n^{(1)})$.
This relevance measures the information flow passing through $n^{(1)}$, which corresponds to many walks, as shown \Cref{fig:A.JointMarginalConditionalRelevance} (b).  
If the relevance is additive with respect to walks, $R(n^{(1)})$ can be computed by marginalizing out the layers in which no neuron is selected, i.e.,
\begin{align}
R(n^{(1)})
&=  \sum_{n^{(0)}, n^{(2)}, n^{(3)}, n^{(4)}}
R(n^{(0)}, n^{(1)}, n^{(2)}, n^{(3)}, n^{(4)}).
   \label{eq:Motivation.Marginal}
\end{align}
The relation between the walk relevance \eqref{eq:Motivation.Joint} and the marginalized relevance \eqref{eq:Motivation.Marginal} is exactly the same as the relation between the joint probability and the marginal probability.   Accordingly we call them joint relevance and marginal relevance, respectively.

We can similarly define the conditional relevance, for example, 
\begin{align}
R(n^{(1)} | n^{(3)})
&=  \textstyle 
\begin{cases}
        \frac{ R(n^{(1)}, n^{(3)} )}{R( n^{(3)})}& \mbox{ if } R(n^{(3)}) \neq 0, \\
        0 & \mbox{otherwise},
     \end{cases}
   \label{eq:Motivation.Conditional}
\end{align}
which measures the proportion of the information flow passing through both $n^{(1)}$ and $n^{(3)}$ among the information flow passing through $n^{(3)}$ (see \Cref{fig:A.JointMarginalConditionalRelevance} (c)).

Joint, marginal, and conditional relevance can also be defined for neuron sets such that  multiple neurons in the same layer can be selected.  For example, we can define
\begin{align}
& R(\mcS^{(2:3)}) \equiv R(n^{(2)} \in \mcS^{(2)}, n^{(3)} \in \mcS^{(3)})
= \sum_{\substack{n^{(0)}, n^{(1)}, \\n^{(4)}}}
\;\;\; \sum_{\substack{ n^{(2)} \in \mcS^{(2)}, \\ n^{(3)} \in \mcS^{(3)}}}
R(n^{(0)}, n^{(1)}, n^{(2)}, n^{(3)}, n^{(4)}),
   \label{eq:Motivation.Marginal.Set}
\end{align}
where $\mcS^{(2:3)} = (\mcS^{(2)}, \mcS^{(3)})$ are sets of neurons in layers 2 and 3, respectively, and the summation for the layers 0, 1, 4 is taken over all neurons in each layer, as in Eq.~\eqref{eq:Motivation.Marginal}.
This relevance quantifies the importance of information flow passing through the substructure $\mcS^{(2:3)}$, as visualized in \Cref{fig:A.JointMarginalConditionalRelevance} (d).

In our NRM framework, one can define the relevance of an arbitrary set of neurons flexibly by the marginalization and conditioning operations, as seen above.  Furthermore, once we define the target relevance quantity, it can be efficiently computed using sum-product message passing algorithms \citep{bishop2006patternchp8}. 
Overall, NRM provides a flexible and efficient tool for explaining interactions between neurons across layers.

The remainder of this section is devoted to a more general mathematical introduction to the NRM framework.

\subsection{Relevance as normalized signed measure}
\label{sec:NormalizedSignedMeasure}

We introduce the NRM framewrok for a restricted form of feed forward neural network (FFNNs)---called a \emph{proper} FFNNs---where all input features lie in the input layer, all output variables lie in the output layer, and no skip connection exists.  However, we can transform any complicated architecture into a proper FFNN, e.g., by unfolding recurrent NN, and copying intermediate input variables to the precedent layers, as detailed in \ref{sec:GeneralNNAdaptation}.  Therefore, it can be applied to NNs with arbitrary architecture.

Let us consider a proper $L$-layer FFNN to be a function $\bff: \mathbb{R}^{\numneu{\il}} \mapsto \mathbb{R}^{\numneu{\ol}}$ that maps from an $N^{(0)}$-dimensional input vector to a $N^{(L)}$-dimensional output vector. For each intermediate layer 
$l = 1, \dots, L-1 $, there is a latent vector that consists of $N^{(l)}$ neurons. Let $\mathbb{N}^{(l)} \equiv \{ 1, \dots, N^{(l)} \}$ denote the set of neuron indices in layer $l$.

We define a function $R$ that quantifies the relevance of any set of neurons or features in the model.  
Our framework is based on the idea that the relevance of a set of neurons should be measured by the information flow passing through the specified neurons.
Therefore, the smallest units into which interacting features of a neural network can be decomposed are walks.
Accordingly, we first define the relevance of walks 
in an intuitive way. By ``intuitive'' we mean that, for two disjoint sets of walks, the relevance of their union equals the sum of their individual relevances, 
that the empty set has zero relevance,
and
that the total relevance over all walks equals a fixed value---here, one. This behavior parallels how measures such as weight or volume are assigned to physical objects.

Formally, we define a \emph{relevance measure} in the following way:
\begin{definition} (Relevance of Set of Walks)
\label{dfn:RelevanceMeasure}
    Let  $\mathbb{W} \equiv   \otimes_{l = 0}^L \, \mathbb{N}^{(l)}$ be the Cartesian product of neuron indices in each layer of the FFNN, so that each element $ \bfn \in \mathbb{W}$ is a walk,
    i.e., $\bfn = (n^{(0)}, \dots, n^{(L)})$ with $n^{(l)} \in \mathbb{N}^{(l)}$. Let $\mcP(\mathbb{W})$ be the power set, i.e., the set of all subsets, of $\mathbb{W}$.
    We specify a \emph{relevance measure} to be    
    a set function $R : \mcP(\mathbb{W}) \rightarrow \mathbb{R}$ such that the following properties hold:
    \begin{align}
    R^{\mathrm{Walk}}( \emptyset) & = 0, \notag  \\
    R^{\mathrm{Walk}}( \mathbb{W})&  = 1,  &\text{(normalization)}
    \label{eq:RelevanceDefinitionNormalization}\\
\textstyle
R^{\mathrm{Walk}}(\bigcup_{j} \mcW_j)
&=\textstyle
\sum_{j} R^{\mathrm{Walk}}(\mcW_j), &\text{(additivity)}
    \label{eq:RelevanceDefinitionAdditivity}
\end{align}
    for any pairwise disjoint sets $\mcW_j \in \mathcal{P}(\mathbb{W})$.%
    \footnote{
With the relevance measure defined as \Cref{dfn:RelevanceMeasure}, the triple $(\mathbb{W}, \mathbb{P} (\mathbb{W}), R)$ forms a (normalized) signed measure space.
    }
\end{definition}

\newcolumntype{K}[1]{>{\raggedright\arraybackslash}p{#1}}

For explaining a model, we are mostly interested in the sets of walks that can be expressed as the Cartesian product $\mcW = \otimes_{l = 0}^L \, \mcS^{(l)} \in \mcP(\mathbb{W})$ of sets of neurons $\mcS^{(l)} \subset \mathbb{N}^{(l)}$ in each layer.   Specifically, we interpret the relevance of those walk sets as the joint relevance of neuron sets.

\begin{definition} (Joint Relevance of Set of Neurons)
\label{dfn:RelevanceNeurons}
Let $\mathbb{L} \equiv \{ 0, \dots, L \}$.
For a subset of layer indices $\mcL =\{l_1, \ldots, l_{|\mcL|}\} \subset \mathbb{L}$, 
we denote a sequence of neuron sets by $\mcS^{(\mcL)} = (\mcS^{(l)})_{l \in \mcL}$.
We identify the joint relevance of the set of neurons with the relevance of all walks that pass through the specified sets, 
i.e.,
\begin{align}
R(\mathcal{S}^{(\mcL)} )
 &=
R^{\mathrm{Walk}}(\otimes_{l = 0}^L \, \widetilde{\mcS}^{(l)}),
 \label{eq:A.NeuronRelevanceDefinition}\\
 \mbox{ where } & \quad
\widetilde{\mcS}^{(l)}
 = \textstyle \begin{cases}
  \mcS^{(l)} & \mbox{ for } l \in \mcL, \\
  \mathbb{N}^{(l)}  & \mbox{ for } l \notin \mcL.
\end{cases}
\label{eq:A.STildeDefinition}
\end{align}
\end{definition}

This definition of joint relevance closely parallels that of joint probability in probability
theory. Specifically, Eq.~\eqref{eq:A.NeuronRelevanceDefinition} reduces to the definition of joint probability if walks and
layers are replaced by events and random variables, respectively.
In this analogy, 
the probability of a set of \emph{events} is expressed in terms of a set of \emph{specified variables}, while all unspecified variables are implicitly \emph{marginalized out}.
Due to this analogy, it is convenient to use notation similar to the probability theory, and we interchangeably express the joint relevance as
\begin{align}
R(\mathcal{S}^{(\mcL)} )
=
R(\bfn^{(\mcL)} \in \mathcal{S}^{(\mcL)} )
=
R(n^{(l_1)} \in \mcS^{(l_1)}, \ldots, n^{(l_{|\mcL|})} \in \mcS^{(l_{|\mcL|})} )
&=
R(\mcS^{(l_1)}, \ldots, \mcS^{(l_{|\mcL|})} ).
\notag
\end{align}

Using Eq.~\eqref{eq:A.STildeDefinition}, which makes the implicit marginalization explicit,
the relevance of a set of neurons can be computed by summing up walk relevances:
\begin{align}
    R(\mcS^{(\mcL)})
& = \sum_{ \bfn \in \mathbb{W} :\, n^{(l)} \in \widetilde{\mcS}^{(l)}  \forall l \in \mcL } R(\bfn),
\label{eq:NeuronRelevanceMarginalization}
\end{align}
where $R(\bfn) = R(n^{(0)}, \ldots, n^{(L)})$ is the relevance of a single walk---which we call the full joint relevance.
Eq.~\eqref{eq:NeuronRelevanceMarginalization} implies that, once the relevance is defined for all individual walks, the relevance of any set of walks is uniquely determined---a direct consequence of the additivity \eqref{eq:RelevanceDefinitionAdditivity} with respect to walks.

When $\mcS^{(l)}$ in each layer consists of a single neuron $n^{(l)}$, we call the sequence $\mcS^{(\mcL)} = \bfn^{(\mcL)}$
a \emph{partial walk}.  As special instances, the relevance of a single neuron $R(n^{(l_1)})$ expresses the importance of a single neuron, while the joint relevance of two neurons $R(n^{(l_1)}, n^{(l_2)})$ expresses the importance of interactions between two neurons.  Note that the former is a marginal relevance of the latter:
\begin{align}
    R(n^{(l_1)})
    & = \sum_{n^{(l_2)} \in \mathbb{N}^{(l_2)}}R(n^{(l_1)}, n^{(l_2)}).
    \notag
\end{align}
\Cref{fig:A.JointMarginalConditionalRelevance} (a) and (b) visualize the walk relevance and the marginal relevance, respectively, as the corresponding sets of walks. We can see that the marginal relevance is the relevance of all walks that pass through the specified neurons.

If the specified layers are consecutive, i.e., $\mcL = \{ \underline{l}, \underline{l}+1, \dots, \overline{l}-1, \overline{l} \}$ 
for some $\underline{l} < \overline{l}$, we also write $\mcS^{(\underline{l}:\overline{l})}$ for the sequence of neuron sets $(\mcS^{(\underline{l})}, \dots, \mcS^{(\overline{l})})$, which we call a \emph{substructure}.
We call the corresponding relevance \emph{substructure relevance}:
 \begin{align}
R( \mcS^{(\underline{l}:\overline{l})})
  & =  
 \sum_{n^{(\underline{l})} \in \subg{\underline{l}}}   \cdots \sum_{n^{(\overline{l})} \in \subg{\overline{l}}} 
    R(n^{(\underline{l})}, \ldots, n^{(\overline{l})} ),
  \label{eq:A.SubstructureRelevance}
  \end{align}
  which is visualized in
  \Cref{fig:A.JointMarginalConditionalRelevance} (d).

We also define the conditional relevance in a similar way to the probability. Let $\mcL_1 , \mcL_2 \subseteq \mathbb{L}$ be disjoint sets of layer indices such that  $\mcL_1 \cap \mcL_2 = \emptyset$.  Then we define the conditional relevance 
as
\begin{align}
     R(\bfn^{(\mcL_1)} | \bfn^{(\mcL_2)})
     &= 
     \begin{cases}
        \frac{ R(\bfn^{(\mcL_1 \cup \mcL_2)} )}{R( \bfn^{(\mcL_2)})}& \mbox{ if } R(\bfn^{(\mcL_2)}) \neq 0, \\
        0 & \mbox{otherwise},
     \end{cases}
\end{align}
which is illustrated in 
\Cref{fig:A.JointMarginalConditionalRelevance} (c).
For a simple instance, $R(n^{(0)} | n^{(L)})$ expresses the importance of an input neuron $n^{(0)}$ to explain the output $n^{(L)}$.

The joint relevance, defined in %
\Cref{dfn:RelevanceNeurons}, shares the same properties as the probability measure, except the non-negativity.  Specifically, it fulfills
\begin{align}
    R(\mcS^{(l)} ) &= 1- R( \overline{\mcS}^{(l)} ), \mbox{ where }  \overline{\mcS}^{(l)} = \mathbb{N}^{(l)} \setminus \mcS^{(l)},
    \notag\\
    R(n^{(l_1)}, n^{(l_2)}) &= R(n^{(l_1)}| \, n^{(l_2)}) R(n^{(l_2)}),
    \notag\\
    R(\mcS_1^{(l)} \cup \mcS_2^{(l)}) &= R( \mcS_1^{(l)} ) +R( \mcS_2^{(l)} )  - R(\mcS_1^{(l)} \cap \mcS_2^{(l)}),
    \notag
    \end{align}
which are called the \emph{complementary law}, \emph{multiplication law}, and \emph{addition law}, respectively.
We collectively call those properties \emph{relevance laws} in analogy to the probability laws. 
For reader's convenience, we
provide a summary table of the relevance definitions and  relevance laws in \ref{sec:A.RelevanceDefinisionAndLaws}.

\subsection{Walk relevance specification}
\label{sec:WalkRelevanceSpecification}

The NRM framework is complete when we specify the relevances of all individual walks for a given neural network and propagation rules.
Although the relevance of a walk is already an established notion \citep{schnake2020higher,DBLP:conf/icml/XiongSMMN22,DBLP:conf/icml/XiongSGMMN23}, we define it differently for generality.  We clarify the relation between our definition and the original definition of walk relevance
in \Cref{sec:GeneralityNRM}.

Since proper FFNNs have no skip connections, it is reasonable to assume the Markov property:
  \begin{align}
 \Rel(n^{(\il)}, \ldots, n^{(\ol)}) & =\textstyle  \left(\prod_{l=1}^{L}  \Rel(n^{(l-1)} | n^{(l)}) \right)   R(n^{(\ol)}).
 \label{eq:JointRelevanceDecomposition}
 \end{align}
Given a test input $\bfx \in \mathbb{R}^{\numneu{\il}}$,
we set an
\emph{unnormalized output relevance vector} based on the network output $\bff(\bfx)$; for example,  
\begin{align}    
\widetilde{\bfr}^{(\ol)} = \bff(\bfx) \in \mathbb{R}^{\numneu{\ol}}.
\label{eq:OutputRelevanceExample}
\end{align}
We also set \emph{unnormalized propagation matrices}
 $\{\widetilde{\bfT}^{(l)} \in \mathbb{R}^{\numneu{l-1} \times \numneu{l}}\}_{l=1}^L$, depending on the choice of LRP rules \citep{bach2015pixel,montavon2018methods,DBLP:journals/pieee/SamekMLAM21}.
    For example, choosing the LRP-$\gamma$ rule with $\gamma \geq 0$, we set 
\begin{align}
 \widetilde{T}^{(l)}_{n, n'} = 
 h^{(l-1)}_{n} W^{(l)\uparrow}_{n, n'} + \varepsilon,
\label{eq:UnnormalizedPropagationMatrixGamma}
\end{align}
where $\bfh^{(l)} \in \mathbb{R}^{\numneu{l}}$ is the activation at the $l$-th layer,
and 
$ \bfW ^{(l)\uparrow} \in \mathbb{R}^{\numneu{l-1} \times \numneu{l}}$ is a modified network weight parameter at the $l$-th layer such that
$\bfW ^{\uparrow} := \bfW  + \gamma\cdot \max(0, \bfW  )$
with the maximization operator applying entry-wise.  
Here, $\varepsilon > 0$ is a small constant for stabilization.

Assuming 
that the unnormalized propagation matrices and output relevance vector  satisfy the \emph{positive-sum conditions},
\begin{align}
\textstyle
 \sum_{n = 1}^{\numneu{l-1}} \widetilde{T}^{(l)}_{n, n'} 
&>0,  
\qquad
 \forall n' = 1, \ldots,  \numneu{l},  \forall l = 1, \ldots, L,
\label{eq:PositiveSumConditionPropagationMatrix} \\
 \textstyle
   \sum_{n =1}^{N^{(\ol)}}
 \widetilde{r}^{(\ol)}_{n} & > 0,
\label{eq:PositiveSumConditionOutputRelevance} 
\end{align}
we define the consecutive conditional relevance and the output relevance by normalizing them:
\begin{align}
R(n^{(l-1)} | n^{(l)})
 &  = \textstyle
  \frac{   \widetilde{T}^{(l)}_{n^{(l-1)}, n^{(l)}}  }{  \sum_{n'^{(l-1)}}    \widetilde{T}^{(l)}_{n'^{(l-1)}, n^{(l)}}   } ,
 \label{eq:ConsecutiveConditionalRelevanceContruction}\\
 R(n^{(\ol)})
  & = \textstyle
  \frac{
\widetilde{r}^{(\ol)}_{n^{(\ol)}}
  }
  {
   \sum_{n'^{(\ol)}}
\widetilde{r}^{(\ol)}_{n'^{(\ol)}}
  }.
 \label{eq:OutputRelevanceContruction}
 \end{align}
 With Eqs.\eqref{eq:ConsecutiveConditionalRelevanceContruction} and \eqref{eq:OutputRelevanceContruction}, 
Eq.~\eqref{eq:JointRelevanceDecomposition} gives the explicit definition of the relevance for any walk.

\paragraph{Multiple variables in each layer}
NNs with complex architectures consist of neurons that express multiple ``variables'' in each layer (see an example case with RNN discussed in \Cref{sec:NormalizationProperty}).
In such cases, one can independently choose LRP-rules for each pair of input and output variables.
 Assume that the $(l-1)$-th layer consists of the neurons expressing variables $\{\bfmu^{(l-1, j)} \in \mathbb{R}^{N^{(l-1, j)}}\}_{j=1}^J$ and the $l$-th layer consists of the neurons expressing variables $\{\bfnu^{(l, k)} \in \mathbb{R}^{N^{(l, k)}} \}_{k=1}^{K}$ with the forward process in the following form:
 \begin{align}
\bfnu^{(l, k)} =    \bfg^{(k)} (\{ \bfW^{(l, j \to k)\T} \bfmu^{(l-1, j)} \}_{j=1}^J) 
\label{eq:ComplicatedForwardProcessGeneral}
 \end{align}
 for $k = 1, \ldots, K$.
Here $\bfW^{(l, j \to k)} \in \mathbb{R}^{N^{(l-1, j)}\times N^{(l, k)} }$ is the network weight parameter matrix responsible for the conversion from $\bfmu^{(l-1, j)}$ to $\bfnu^{(l, k)}$  (see \Cref{fig:MultipleVariables}).
In this case, we first define unnormalized relevance propagation matrices $\{\widetilde{\bfT}^{(l, j \leftarrow k)} \in \mathbb{R}^{N^{(l-1, j)}\times N^{(l, k)} } \}$ between each pair of variables.  For example, if we apply the LRP-$\gamma$ rule, we set
\begin{align}
 \widetilde{T}_{n, n'}^{(l, j \leftarrow k)} = 
 \mu^{(l-1, j)}_{n} W^{(l, j \to k)\uparrow}_{n, n'} + \varepsilon,
\label{eq:UnnormalizedPropagationMatrix}
\end{align}
for $j = 1, \ldots, J$, and $k = 1, \ldots, K$.
Then, the complete unnormalized propagation matrix $\widetilde{\bfT}^{(l)} \in \mathbb{R}^{N^{(l-1)}\times N^{(l)} } $, where $N^{(l-1)} =\sum_{j=1}^J N^{(l-1, j)}$
and $N^{(l)} =\sum_{k=1}^K N^{(l, k)}$, can be defined as
\begin{figure}[t]
    \centering
\qquad\qquad
    \begin{subfigure}[t]{0.2\linewidth}
        \centering
        \includegraphics[width=\linewidth]{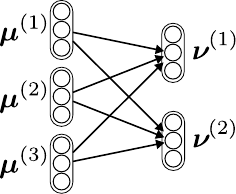}
        \caption{}
        \label{fig:MultipleVariables}
    \end{subfigure}
    \qquad \qquad
    \begin{subfigure}[t]{0.5\linewidth}
        \centering
        \includegraphics[width=\linewidth]{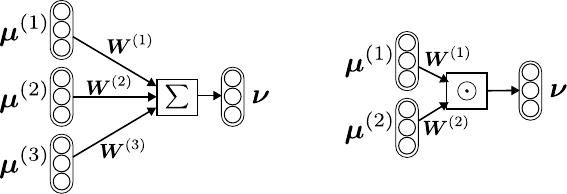}
        \caption{}
        \label{fig:MultipleVariablesAdditiveMultiplicative}
    \end{subfigure}
    \vskip -1em
    \caption{(a) Multiple input and output variables. (b) Additive (left) and multiplicative (right) forward processes.}
    \label{fig:ForwardProcesses}
\end{figure}
\begin{align}
    \widetilde{\bfT}^{(l)}
    & = 
    \begin{pmatrix}
      \alpha^{(l, 1 \leftarrow 1)}  \widetilde{\bfT}^{(l, 1 \leftarrow 1)} & \cdots &    \alpha^{(l, 1 \leftarrow k )}  \widetilde{\bfT}^{(l, 1 \leftarrow K)} \\
        \vdots &\ddots  & \vdots \\
     \alpha^{(l, J \leftarrow 1)}   \widetilde{\bfT}^{(l, J \leftarrow 1)} & \cdots &  \alpha^{(l, J \leftarrow K)}\widetilde{\bfT}^{(l, J \leftarrow K)}     \end{pmatrix},
        \label{eq:WholeUnnormalizedPropagationMatrix}
\end{align}
where $\{  \alpha^{(l,  j \leftarrow k )} \in \mathbb{R}\} $ are controllable coefficients specifying the relevance distribution from the output to the input.
Those coefficients are usually set to $\alpha^{(l, j \leftarrow k)}  = 1$ if the input contributions are additive (see \Cref{fig:MultipleVariablesAdditiveMultiplicative} left) in the forward process \eqref{eq:ComplicatedForwardProcessGeneral}, i.e.,
 \begin{align}
\bfnu^{(l, k)} =  \textstyle  \bfg^{(k)} \left(\sum_{j=1}^J  \bfW^{(l, j \to k)\T} \bfmu^{(l-1, j)} \right) 
\label{eq:ComplicatedForwardProcessAdditive}
 \end{align}
 for $k = 1, \ldots, K$.
 In contrast, they are typically set to extreme values to ignore some inputs if multiplications are involved (see \Cref{fig:MultipleVariablesAdditiveMultiplicative} right).  For example, LRP-all rule \citep{DBLP:series/lncs/ArrasAWMGMHS19} propagates no relevance to the attention neurons.  Namely, it sets to $\alpha^{(l, 1 \leftarrow k )} = 0, \alpha^{(l, 2 \leftarrow k)} = 1$,
if the forward process is 
 \begin{align}
\bfnu^{(l, k)}& =    \bfg^{(k)} \Big(
\sigma \left(\bfW^{(l, 1 \to k)\T}  \bfmu^{(l-1, 1)} \right)
\odot
\left(\bfW^{(l, 2 \to k)\T} \bfmu^{(l-1, 2)}  \right)
\Big) ,
\label{eq:ComplicatedForwardProcessMultiplicative}
 \end{align}
where $\bfW^{(l, 1 \to k)\T}  \bfmu^{(l-1, 1)}$ provides the attention through the (element-wise) sigmoid activation $\sigma(\cdot)$,
and $\odot$ denotes the entry-wise product of vector or matrix.
With the whole matrix \eqref{eq:WholeUnnormalizedPropagationMatrix} defined, 
the consecutive conditional relevance is given by 
Eq.~\eqref{eq:ConsecutiveConditionalRelevanceContruction} for the $l$-th  layer.

Defining the propagation matrices for all layers, as well as the output relevance \eqref{eq:OutputRelevanceContruction}, completes the full joint relevance definition \eqref{eq:JointRelevanceDecomposition}.

\subsection{LRP as message passing}
\label{sec:MessagePassingAlgorithms}

As pointed out by \citet{DBLP:conf/icml/XiongSMMN22},
the joint relevance \eqref{eq:JointRelevanceDecomposition}
has the same factorizabilty to the probability of a Markov chain, and therefore,
one can efficiently compute the relevance of a set of neurons by \emph{sum-product message passing} \citep{bishop2006patternchp8}.
As a generalition of the sum-product algorithms proposed in \citet{DBLP:conf/icml/XiongSMMN22}, the following theorem provides LRP algorithms for attributing arbitrary sets of neurons (proof is given in \ref{sec:A.ProofLRPasMessagePassing}):
\begin{theorem} (LRP as Message Passing)
\label{thrm:LRPasMessagePassing}
Assume that, given a specified neuron set $\mcS^{(\mcL)} = (\mcS^{(l)})_{l \in \mcL}$, we performed the message passing,
\begin{align}
\msgbel{l-1}_{n^{(l-1)}} & = \textstyle \sum_{n^{(l)} \in \widetilde{\mcS}^{(l)}}    R(n^{(l-1)}| n^{(l)})  \msgbel{l}_{n^{(l)}}  ,
\label{eq:A.LRPasMP}
\end{align}
for $l = L, \ldots, l_1+1$,
where
$l_1 = \min(\mcL)$,
$\widetilde{\mcS}^{(l)}$ is defined in Eq.~\eqref{eq:A.STildeDefinition},
and
$\msgb{l}  \in \mathbb{R}^{\numneu{l}}$ is the message at the $l$-th layer
with the initial message $\msgbel{L}_{n^{(L)}} = R(n^{(L)})$ 
being the output relevance.
Then, the final message gives 
\begin{align}
\Rel(\mcS^{(\mcL)})
&=\textstyle
\sum_{n^{(l_1)} \in \widetilde{\mcS}^{(l_1)}}
\msgbel{l_1}_{n^{(l_1)}}.
\label{eq:A.LRPasMP.Result}
\end{align}
\end{theorem}

The LRP algorithm \eqref{eq:A.LRPasMP}
can compute the relevance of any set of neurons $\mcS^{(\mcL)}$.
Conversely, 
since Eq.~\eqref{eq:A.LRPasMP} is the general form of LRP algorithm for any proper FFNN, the NRM framework guarantees that any LRP algorithm with normalized propagation matrices and output relevance computes the normalized relevance of some set of neurons.
In \Cref{sec:ExistingLRPinNRM}, 
we recast existing LRP algorithms within the NRM framework, and identify the target relevance quantities they compute.

\subsection{Generality and extensions}
\label{sec:GeneralityNRM}

As discussed in the previous subsections, our NRM framework is general in the sense that it can attribute any set of neurons in any network architecture.
In this section, we show that the framework is indeed a generalization of existing propagation-based XAI methods.
We further discuss an alternative definition of relevance for neuron sets with different properties, and 
the possibility to extend the framework to occlusion- and gradient-based XAI methods.

\subsubsection{Revisiting existing LRP algorithms with NRM}
\label{sec:ExistingLRPinNRM}

We start from clarifying the relation between the full joint relevance---the basic unit of relevance attribution in our NRM framework---and the original definition of walk relevance \citep{schnake2020higher}.
Given unnormalized propagation matrices and output relevance, the original walk relevance is defined as
 \begin{align}
 \widetilde{r}_{\bfn}
=
 \widetilde{r}_{(n^{(\il)}, \ldots, n^{(\ol)})} &= \left( \textstyle \prod_{l=1}^{L}   \widetilde{T}^{(l)}_{n^{(l-1)}, n^{(l)}} \right)  \widetilde{r}^{(\ol)}_{n^{(\ol)}},
 \label{eq:OriginalWalkRelevance}
 \end{align}
 which is considered to be an approximate activation flow of FFNNs, according to the deep Taylor decomposition interpretation \citep{DBLP:journals/pr/MontavonLBSM17}.
The following theorem holds
(proof is given in \ref{sec:A.ProofWalkRelevanceRelation}):
\begin{theorem} 
  (Relation to original walk relevance definition)
\label{thrm:EquivalenceToOriginalWalkRelevance}
Assume that the unnormalized propagation matrices and output relevance satisfy the constant-sum conditions:
\begin{align}
\textstyle \sum_{n = 1}^{\numneu{l-1}} \widetilde{T}^{(l)}_{n, n'} 
& =  \fconst{l}, 
\quad
\mbox{ for some } \{\fconst{l} > 0 \}_{l = 1}^{L}
, \forall n' = 1, \ldots,  \numneu{l}, \forall l = 1, \ldots, L,
\label{eq:A.ConstantSumConditionPropagationMatrix}\\
 \textstyle
  \sum_{n = 1}^{N^{(\ol)}}
 \widetilde{r}^{(\ol)}_{n} & > 0.    \label{eq:A.ConstantSumConditionOutputRelevance}
\end{align}
Then, the full joint relevance in the NRM framework coincides with the normalized version of the original walk relevance, i.e.,
 \begin{align}
 \Rel(\bfn) 
  &= 
\Rel(n^{(\il)}, \ldots, n^{(\ol)}) 
 =
 \frac{ \widetilde{r}_{\bfn}}
{
\sum_{\bfn' \in \mathbb{W}}
 \widetilde{r}_{\bfn'}}.
 \label{eq:A.JointRelevanceAndWalkRelevance}
 \end{align}
\end{theorem}

Since Eqs.\eqref{eq:ConsecutiveConditionalRelevanceContruction}, \eqref{eq:OutputRelevanceContruction}, and \eqref{eq:A.JointRelevanceAndWalkRelevance} are normalization operators, it trivially holds that
\begin{align} 
R(n^{(l-1)} | n^{(l)}) 
& = \widetilde{T}^{(l)}_{n^{(l-1)}, n^{(l)}},
\quad
R(n^{(\ol)}) 
=  \widetilde{r}^{(\ol)}_{n^{(\ol)}},
\quad
\Rel(\bfn)
=
  \widetilde{r}_{\bfn},
\label{eq:A.EquivalanceNormalizedWalkRelevance}
\end{align}
if the propagation matrices and output relevance are normalized, i.e., 
\begin{align}
\textstyle
\sum_{n=1}^{N^{(l-1)}}\widetilde{T}^{(l)}_{n,n'} 
& = 
C^{(l)} = 1,  
\qquad
\forall n' = 1, \ldots,  \numneu{l}, 
\forall l = 1, \ldots, L, 
\label{eq:A.NormalizedConditionPropagationMatrix}\\
 \textstyle
 \sum_{n=1}^{N^{(\ol)}}
 \widetilde{r}^{(\ol)}_{n} & = 1.
 \label{eq:A.NormalizedConditionOutputRelevance}
 \end{align}

For most common LRP rules, the propagation matrices are defined in (approximately) normalized forms (see \ref{sec:A.CommonLRPRules}), and thus
satisfy the normalization condition \eqref{eq:A.NormalizedConditionPropagationMatrix}. 
Although the output relevance is usually not normalized, its normalization only affects the absolute scale of relevance, which is not essential in typical XAI scenarios, e.g., where heat-maps are depicted or relevant features are ranked.

\begin{table}[t]
\caption{NRM expression of the target relevance quantities that existing LRP algorithms compute. 
}
\vspace{-4mm}
\center
\begin{tabular} {K{.4\textwidth}K{.5\textwidth}}
\toprule
LRP Algorithm
& Target relevance quantity \\
\midrule
Standard LRP \cite{bach2015pixel} & $ R(n^{(\il)},  {n}^{(\ol)})$ \\
ConRel \cite{ConveptRelevancePropagation} & $ R(n^{(l)}, n^{(\ol)})$ \\
CRP \cite{ConveptRelevancePropagation} & $R(n^{(\il)}, n^{(l)},  n^{(\ol)})$\\
Subgraph LRP \cite{DBLP:conf/icml/XiongSMMN22} & $R(n^{(\underline{l}:\overline{l})} \in \mcS^{(\underline{l}:\overline{l})} | n^{(\ol)})$\\
\bottomrule
\end{tabular}
\label{table:LRPinNRFramework.Small}
\end{table}

Assuming the normalization conditions \eqref{eq:A.NormalizedConditionPropagationMatrix} and \eqref{eq:A.NormalizedConditionOutputRelevance},
we can straightforwardly recast existing LRP methods within the NRM framework, and identify the quantities they compute, as summarized in \Cref{table:LRPinNRFramework.Small}.
Here, the standard LRP \citep{bach2015pixel,DBLP:journals/pieee/SamekMLAM21} is the original LRP algorithm to explain a chosen output $n^{(L)}$ by attributing relevance to each input neuron $n^{(0)}$, Concept Relevance (ConRel), which is the objective function for Relevance Maximization (RelMax) \citep{ConveptRelevancePropagation},  attributes relevance to an intermediate neuron $n^{(l)}$ for $0 < l < L$,
Concept Relevance Propagation (CRP) \citep{ConveptRelevancePropagation} explains the contribution of an intermediate neuron $n^{(l)}$ by attributing relevance to each input neuron $n^{(0)}$,
and GNN-LRP \citep{schnake2020higher,DBLP:conf/icml/XiongSMMN22}
attributes relevance to a substructure $\mcS^{(\underline{l}:\overline{l})}$ within a GNN. 
In \ref{sec:A.RalationToLRPAlgorithms}, we provide detailed descriptions of those LRP algorithms,  derivations of the target relevance quantities, and illustrations of the relevance quantities as sets of walks.
We furtheremore recast LRP methods for attributing neuron sets, including Symbolic XAI (SymbXAI) \citep{SCHNAKE2025}, in the NRM framework.

\subsubsection{Additive relevance definition for neuron sets}
\label{sec:AdditiveRelevanceDefinition}

As discussed in the previous sections, the joint relevance definition (\Cref{dfn:RelevanceNeurons}) enables the analysis of information flows that pass through a specified set of neurons,
and 
the relevance of any such set can be defined through simple marginalization and conditioning operations. Furthermore, the joint relevance establishes a connection between message passing algorithms for probability computation and the LRP algorithms for XAI. 
However, 
the normalization and additivity hold \emph{layer-wise}, and adding a neuron from a new layer decreases the joint relevance, analogous to the decrease of joint probability as more variables are included.
As a result, the joint relevance does not provide an \emph{additive} decomopsition of the output relevance over the neurons in the whole network.
Here, we briefly discuss an alternative relevance definition in which relevance forms a normalized signed measure with respect to (not only walks but also) neurons.

Let $\mathbb{S}$ be the set of all neurons in the network considered.
Based on the walk relevance definition (\Cref{dfn:RelevanceMeasure}),
a general definition of the relevance of a set $\mcS \subseteq \mathbb{S}$ of neurons can be given as
    \begin{align}
    R^{\mathrm{Neu}}(\mcS)
& =  \sum_{ \bfn \in \mathbb{W}} R^{\mathrm{Walk}}(\{\bfn\}) \phi(\bfn, \mcS),
\label{eq:GeneralNeuronSetRelevance}
\end{align}
where $\phi: \mathbb{W} \times \mathbb{S} \mapsto  \mathbb{R}$ is a weighting function that determines the properties of the relevance.  Note that the joint relevance definition (\Cref{dfn:RelevanceNeurons}) is recovered by setting 
$\phi(\bfn, \mcS^{(\mcL)}) = \mathbbm{1}(n^{(l)} \in {\mcS}^{(l)},  \forall l \in \mcL )$.

By first distributing walk relevance to each neuron, and then summing them up,
we can define the relevance of sets of neurons such that the additivity with respect to neurons holds.
Let $\mathbbm{1}({\cdot})$ be the indicator function that equals one if the event is true, and zero otherwise.
 \begin{definition} (Additive Relevance of Set of Neurons)
 \label{def:AdditiveRelevanceDefinition}
 We define an additive relevance by the set function \eqref{eq:GeneralNeuronSetRelevance} of neurons with the weight function  
$\phi(\bfn, \mcS) = \frac{
\sum_{n \in \bfn} \mathbbm{1} ({ n \in \mcS) }
}{
|\bfn|
}$,
where  the walk $\bfn$ is seen as a set of neurons.
\end{definition}
\begin{theorem}
\label{thrm:Properties.AdditiveRelevance}
The additive relevance is a signed measure function with respect to the sets of neurons across layers, i.e., it satisfies
\begin{align}
    R^{\mathrm{Add}}( \emptyset)  = 0, 
\qquad
    R^{\mathrm{Add}}( \mathbb{S})  = 1,  
\qquad
 \textstyle
R^{\mathrm{Add}}(\bigcup_{j} \mcS_j)
=\textstyle
\sum_{j} R^{\mathrm{Add}}(\mcS_j), 
    \notag
\end{align}
    for any pairwise disjoint sets $\mcS_j \in \mathcal{P}(\mathbb{S})$.%
\end{theorem}
The proof of \Cref{thrm:Properties.AdditiveRelevance} and a preliminary numerical comparison with the joint relevance definition are given in \ref{sec:A.DetailsAdditiveRelevance}.

Although the joint relevance definition is better suited to our goal in this paper---analyzing interactions among neurons across layers---we believe that the additive relevance definition may also be useful.
In particular, it could enable the decomposition of output relevance into meaningful groups of neurons, such as \emph{circuits}.
A systematic study of additive relevance and its potential applications is left for future work.

\subsubsection{NRM based on perturbation- and gradient-based XAI methods}

To apply the NRM framework to explanation paradigms beyond propagation-based methods, it suffices to specify how to compute the relevance of walks within these paradigms. Intuitively, the relevance of walks captures higher-order contributions of latent features, a concept that is also well established in perturbation- and gradient-based XAI methods.

In the perturbation-based XAI framework \cite{lund2017unified, BLUCHER2022103774}, a straightforward specification of $R(n^{(0)}, \dots, n^{(L)})$ can be obtained by perturbing all neurons except $\{n^{(1)}, \dots, n^{(L)}\}$. The perturbation of latent features can, for example, be performed as in \cite{gorbani2020neuron_shap}.

For gradient-based methods, such as Integrated Gradients \cite{sundararajan2017integrated_gradient}, the relevance of latent features has also been studied (see, e.g., \cite{dhamdhere2018how}). In this case, the walk relevance $R(n^{(0)}, \dots, n^{(L)})$ corresponds to a higher-order derivative of the model with respect to the neurons $n^{(0)}$ through $n^{(L)}$, analogous to the second-order case discussed in \cite{janizek2021integrated_hessian}.

Once the relevances of all individual walks are specified, \Cref{dfn:RelevanceMeasure} defines the relevance of all sets of walks, based on the additivity and normalization requirements, and \Cref{dfn:RelevanceNeurons} defines the joint relevance of any set of neurons. In this work, we focus on specifying walk relevance using propagation-based XAI techniques and leave the exploration and evaluation of alternative approaches for future work.

\subsection{Summary and key features of NRM framework}
\label{sec:KeyFeatures}

Our NRM framework offers an explicit and intuitive procedure for defining relevance of neurons, together with a systematic derivation of corresponding LRP algorithms, with the following key features:
\begin{itemize}
    \item The framework can be applied to any network architecture.
    \item Relevance is defined for \emph{any set of neurons}, allowing for analyzing higher-order interactions across layers.
    \item
    Relevance is defined as a signed measure of a set of walks, allowing for explicit mathematical and visual expressions, as in 
    \Cref{table:LRPinNRFramework.Small}
    and
    \Cref{fig:A.JointMarginalConditionalRelevance}, respectively.
     \item
     LRP algorithms for computing the relevance of specified neurons are systematically derived as sum-product message passing, as in \Cref{thrm:LRPasMessagePassing}. %

         \item Relevance is normalized, and thus guaranteed to be comparable across layers, which will be demonstrated in \Cref{sec:NormalizationProperty}.

\end{itemize}

We give a guide on the procedures of applying NRM on an arbitrary NN in \ref{app:procedure_nrm}.

\section{Experiments}
\label{sec:ExperimentCV}

In this section, we demonstrate the effectiveness of our NRM framework. We first apply NRM to a simple MLP trained on MNIST and conduct quantitative evaluations, showing how accurately the joint relevance captures higher-order interactions in the decision-making process. We then apply the framework to the VGG16 model trained on real-world computer vision data, demonstrating that NRM provides fine-grained and faithful explanations. Finally, we apply NRM to an RNN to illustrate its generality, as well as the importance of the normalization property when feature interactions occur across different layers.

\subsection{Baseline methods}

Since there is no existing method for quantifying the importance of interactions within arbitrary sets of neurons, we compare our NRM approach with the following intuitive baselines.

\paragraph{Standard LRP (LRP): } As a naive approach to quantifying the relevance of a set of neurons, we use the sum of marginal (first-order) relevances, i.e.,
$R^{\mathrm{marginal}}(\mcS^{(\mcL)} | n^{(L)})\propto \sum_{l \in \mcL} R(\mcS^{(l)}| n^{(L)})$. 

\paragraph{Activation (Act): } Inspired by activation based methods such as Activation Maximization  \citep{activationmax}, we use the sum of neuron activations, defined as
$R^{\mathrm{activ}}(\mcS^{(\mcL)}| n^{(L)})\propto \sum_{n\in \mcS^{(\mcL)}}A(n)$, where $A(n)$ is the activation of the neuron $n$ during the forward pass. 

\paragraph{Occlusion (Occ): } A popular approach to identifying important features is to measure the change in the model output when those features are removed. 
Specifically, we quantify the relevance of a set of neurons as $R^{\mathrm{occlusion}}
(\mcS^{(\mcL)}| n^{(L)}) \propto f_{n^{(L)}}(\bfx) - f_{n^{(L)}}^{\mathrm{neu}}(\bfx; \backslash \mcS^{(\mcL)})$, where %
$ f_{n}^{\mathrm{neu}}(\bfx;  \backslash \mathcal S)$ denotes the $n$-th output for input $\bfx$, with the neurons in the set $\mcS$ removed.

\subsection{Quantitative evaluation settings}
\label{sec:quanti_eval_setting}

In quantitative experiments, we evaluate the quality of joint relevance $R(\mcS^{(\mcL)} | n^{(L)})$ (conditioned on the output neuron) of $\mcS^{(\mcL)} = (\mcS^{(l_1)}, \ldots, \mcS^{(l_2)})$ for $0 \leq l_1 < l_2 < L$ specifying an arbitrary set of neurons across layers.  Since the joint relevance quantifies the importance of the walks passing through the neurons, we assess its quality by measuring the consistency with the \emph{joint contribution}:
\begin{align}
        C(\mcS^{(\mcL)}| n^{(L)}) 
    &= \sum_{ \mcB \subseteq \mcL } (-1)^{|\mcB|} f_{n^{(L)}}^{\mathrm{neu}} (\bfx; \backslash \mcS^{(\mcB)}),
    \label{eq:JointContributionGeneral}
\end{align}
which generalizes the interpretable feature \citep{BLUCHER2022103774,SCHNAKE2025} that directly quantifies the contribution of information flow passing through $\mcS^{(\mcL)}$ via interventions on the network output---it can be seen as a pixel-flipping procedure applied to a set of walks.
For second-order explanation, i.e., analyzing interactions between two layers, the joint contribution \eqref{eq:JointContributionGeneral} reduces to
\begin{align}
        C(&\mcS^{(\{l_1, l_2\})}|  n^{(L)}) 
    =f_{n^{(L)}}(\bfx) -f_{n^{(L)}}^{\mathrm{neu}} (\bfx; \backslash \mcS^{(l_1)}) - f_{n^{(L)}}^{\mathrm{neu}} (\bfx; \backslash \mcS^{(l_2)}) + f_{n^{(L)}}^{\mathrm{neu}} (\bfx; \backslash \mcS^{(\{l_1,l_2\})}).
    \label{eq:M.inclusion_exclusion_second_layer}
\end{align}

As an evaluation criterion, we measure the \emph{Pearson correlation} between the joint contribution and the joint relevance computed by each explanation method.
We further measure the sum of joint contributions $\mathrm{SC}_k$ over the top-$k$ most relevant neuron sets as an additional metric.
Detailed description and derivations of these evaluation measures are provided in \ref{sec:eval_methods}.

\subsection{Quantitative validation on multi-layer perceptrons}
\label{sec:Experiment.MLP}

We apply NRM to a ($L=3$)-layered MLP with $28\times28$ input, $256$ and $128$ hidden, and 10 output neurons, trained on MNIST.
We follow the default train-test split provided by \texttt{torchvison}, under which the model achieves a test accuracy of 97.71\%. Since MLP is a proper FFNN, the relevance under the NRM framework is straightforwardly defined, with the full joint relevance (conditioned on an output neuron $n^{(L)}$) expressed as $R(n^{(0)}, n^{(1)}, n^{(2)}| n^{(L)})$. We use the LRP-$0$ rule, i.e., Eq.~\eqref{eq:UnnormalizedPropagationMatrixGamma} with $\gamma=0$ and $\epsilon=0$, for defining the unnormalized propagation matrices, and quantitatively evaluate how accurately the joint relevance captures the second- and third-order joint contributions.

\Cref{tab:res_quanti} (left columns) shows the Pearson correlation between the NRM joint relevance and the joint contribution. The first column is for the second-order joint relevance $R( n^{(1)}, n^{(2)}| n^{(L)})$, while the second column is for the third-order joint relevance $R(n^{(0)}, n^{(1)}, n^{(2)}| n^{(L)})$ including the input layer.  We observe much higher correlation by the NRM joint relevance than the baseline methods.
The correlation is visualized in \ref{sec:AdditionalResults}.

\begin{table}[t]
    \centering
    \begin{tabular}{l|cc|cc}
        &MLP 2nd Order     &MLP 3rd Order      &VGG16 Channel & VGG16 Cluster    \\ \hline
\textbf{NRM}          & \textbf{0.9470} & \textbf{0.4468}    & \textbf{0.5283}     & \textbf{0.3987}      \\
LRP   & 0.4568          & 0.0103     & 0.0005                        & 0.3324          \\
Occ.       & 0.1960          & 0.0084    & 0.0429                       & 0.3438                             \\
Act.          & 0.1513          & 0.0063     & 0.0005                        & 0.1621                         
\end{tabular}
    \caption{
    Pearson correlation with the joint contribution  (higher is better). The results are averaged over the entire test set of MNIST for MLP, and 100 randomly chosen samples from ImageNet for VGG16. 
    }
    \label{tab:res_quanti}
\end{table}

\begin{figure}[t]
    \centering
    \includegraphics[width=0.5\linewidth]{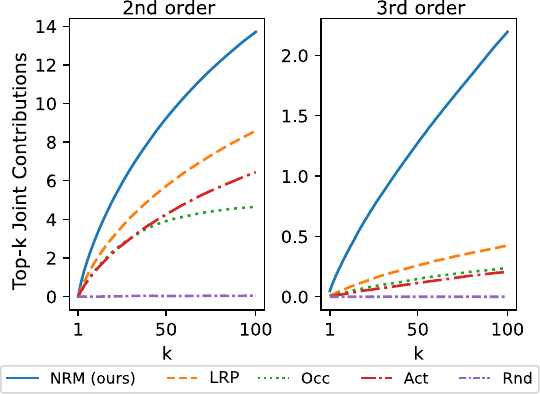}
    \vspace{-3mm}
    \caption{
        Sum of joint contributions over the top-$k$ most relevant neuron sets in MLP  %
        (higher is better), averaged over the entire MNIST test set. 
        ``Rnd'' denotes the random ranking of relevant neuron sets, serving as a reference for the chance-level performance.
    }
    \label{fig:prune_res_mnist}
\end{figure}

\Cref{fig:prune_res_mnist} depicts the second-order (left) and the third-order (right) sums of joint contributions SC$_k$ as a function of $k$.  
This confirms the superiority of the NRM joint relevance to the baselines.

\subsection{The NRM framework on VGG16}
\label{sec:Experiment.VGG16}

Next, we apply NRM to the pretrained VGG16---which is also a proper FFNN---downloaded 
from \texttt{torchvision}. 
To define the unnormalized propagation matrices, we follow \citet{ConveptRelevancePropagation} and apply the LRP-flat rule for the input layer, the LRP-$z+$ rule for the convolution layers, and the LRP-$\varepsilon$ rule for all other layers.

\subsubsection{Joint channel relevance}

We first explain VGG16 with    
the \emph{channel-level}
joint relevance, conditioned on the output neuron $n^{(L)}$ for the true class:
\begin{align}
    R(n^{(l_1)} \in \mcS^{(l_1,k_1)}, n^{(l_2)} \in \mcS^{(l_2,k_2)}, \dots | n^{(L)}),
    \label{eq:JointRelevanceChannel}
\end{align}
which is the relevance of all walks going through the chosen channel in each layer, i.e., channel $k_1$ in layer $l_1$, channel $k_2$ in layer $l_2$, and so on.
The joint relevance quantifies the importance of the interaction between the chosen channels across layers.
We also investigate which pixels in the input image activate this higher-order interaction by drawing the heatmap given by
\begin{align}
    R( n^{(0)},  n^{(l_1)} \in \mcS^{(l_1,k_1)}, n^{(l_2)} \in \mcS^{(l_2,k_2)}, \dots| n^{(L)}).
    \label{eq:JointRelevanceChannelHeatmap}
\end{align}
This is a generalization of CRP \citep{ConveptRelevancePropagation} for explaining a joint concept of a set of intermediate neurons.

\begin{figure}[t]
    \centering
\includegraphics[width=0.6\linewidth]{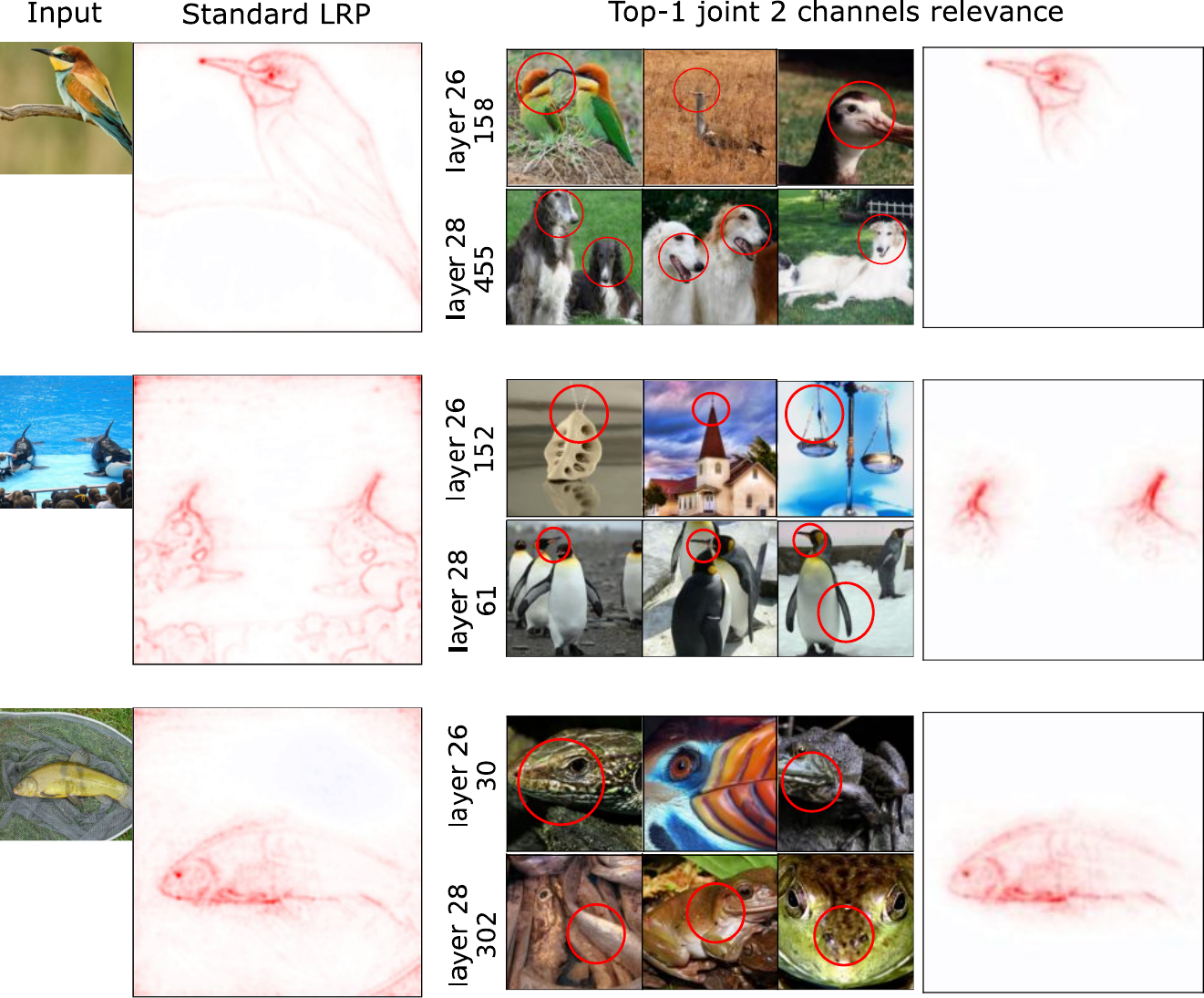}
    \vspace{-2mm}
    \caption{Channel-level joint relevance explanation on three example input images. For each row from left to right: input image, standard LRP heatmap, the set of channels in layers 26 and 28 with highest joint relevance, and the heatmap explaining the interaction.
    In the third column, first row ``Layer $a$'' and second row ``$b$'' means channel $b$ in layer $a$, and the three images shown for each channel are the top-3 RelMax images \citep{ConveptRelevancePropagation} selected from the ImageNet dataset, which illustrate the features that most strongly contribute to the relevance assigned to the channel.        %
    }
    \label{fig:joint_channel}
\end{figure}

\begin{figure}[t]
    \centering
    \includegraphics[width=0.5\linewidth]{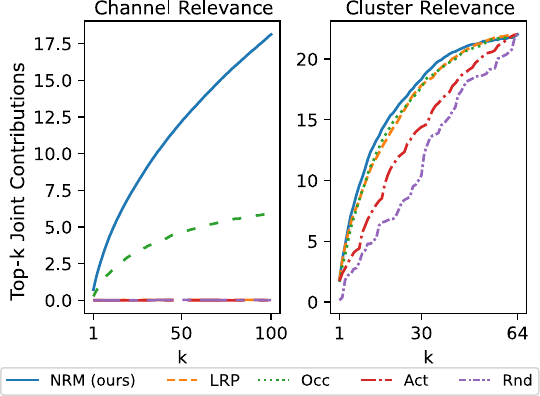}
    \vspace{-3mm}
    \caption{
            Sum of joint contributions over the top-$k$ most relevant neurons sets in VGG16
            (higher is better), averaged over 100 randomly chosen ImageNet samples.
    }
    \label{fig:prune_res}
\end{figure}

\paragraph{Qualitative results}
\Cref{fig:joint_channel} shows results of our joint relevance analysis on three example images of a bird (top), whales (middle), and a tench (bottom).  For each input image (first column), the standard LRP heatmap (second column), the set of channels with the highest joint relevance \eqref{eq:JointRelevanceChannel} in layers 26 and 28 (third column),%
\footnote{We chose those layers, following \citet{ConveptRelevancePropagation}.}
and the heatmap \eqref{eq:JointRelevanceChannelHeatmap} explaining this highest interaction (fourth column) are shown.
We observe that higher-order interactions across layers tend to focus on specific features for the target class.
Namely, from the highest relevant set of channels (third column) and the corresponding heatmap (fourth column), we can see that the model---through the most relevant second order interaction---focuses on  
bird's head (or eye), whales' fins, and fish's body, which are important features for identifying birds, whales, and tenches, respectively. 

Such information is not captured by the standard LRP heatmap (second column), highlighting the value of joint relevance analysis within the NRM framework.
Note that the hierarchical application of the concept relevance analysis \citep{ConveptRelevancePropagation} corresponds to  a greedy search for relevant interactions, which often does not yield the highest joint relevance, and thus not necessarily the most important interaction.

\paragraph{Quantitative results} 
\Cref{tab:res_quanti} (third column) and \Cref{fig:prune_res} (left) show, respectively, the Pearson correlation with the joint contribution, and the sum of joint contributions
over the top-k most relevant neuron sets.
The results indicate that our NRM joint relevance 
substantially outperforms the baseline methods.

\subsubsection{Joint cluster relevance}
\label{sec:cluster_experiment}

In VGG16, it was reported that many channels are highly correlated \citep{DBLP:journals/tnn/KauffmannERMSM24}, which can render channel-level higher-order explanation overly fine-grained.  Here, we aim to uncover bundled higher-order interactions by clustering the spatial features within each layer.
To this end, following \citet{DBLP:journals/tnn/KauffmannERMSM24},
we apply $k$-means clustering \citep{lloyd1982least} to activations in each layer.  More specifically, we treat
an intermediate feature with the shape of $\mathbb R^{h \times h\times c}$ as $h\times h$ spatial samples of $c$-dimensional vectors, and cluster them with $k= 8$.  
We denote by $\mcS^{(l,k)}$ the set of neurons belonging to cluster $k$ in layer $l$.
We analyze the joint cluster relevance in the output layers, denoted by $l_1$ and $l_2$, of blocks 1 and 5, respectively, along with the associated input-pixel explanation:
\begin{align}
  &  R(n^{(l_1)} \in \mcS^{(l_1,k_1)},  n^{(l_2)} \in \mcS^{(l_2,k_2)} | n^{(L)} ),
  &
    R(n^{(0)}, n^{(l_1)} \in \mcS^{(l_1,k_1)},  n^{(l_2)} \in \mcS^{(l_2,k_2)} | n^{(L)}).
    \notag
\end{align}

\paragraph{Qualitative results}
\Cref{fig:joint_cluster} shows the heatmap explanations of the 
cluster-level interactions in layers $l_1$ and $l_2$ corresponding to the top-5 highest joint relevance. The results indicate that the 
cluster-level second-order interactions decompose the entire information flow into distinct semantic parts. Specifically, for the gold fish image, the 1st and 3rd most relevant interactions separately focus on the upper and lower fish bodies, and the 5th focuses on the upper fish's tail and fins. The 2nd and 4th focus on the background.
For the tench image, the top-2 interactions focus on the tench body, and the rest focus on the background grass and net edge.

Again, such information is not available from standard LRP, underscoring the additional explanatory power of the NRM framework.

\paragraph{Quantitative results}

\Cref{tab:res_quanti} (fourth column) and \Cref{fig:prune_res} (right) quantitatively demonstrate the advantage of the NRM joint relevance over the baseline methods.

\begin{figure}[t]
    \centering
    \includegraphics[width=0.6\linewidth]{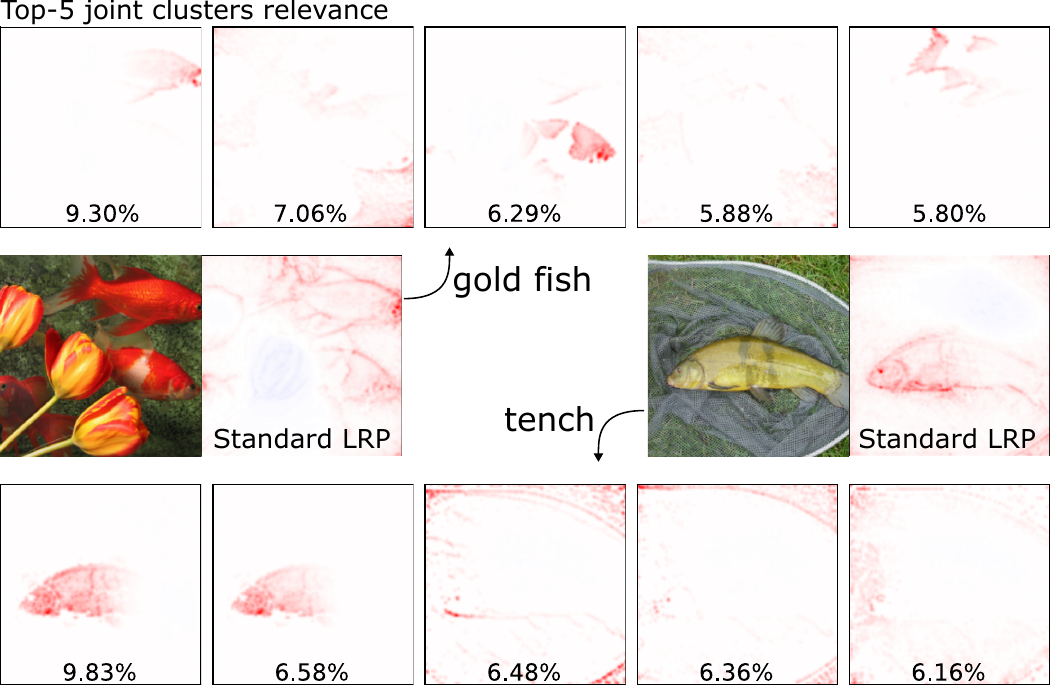}
    \vspace{-2mm}
    \caption{
    Cluster-level joint relevance analysis.
    Top:
    Heatmap explanations of the clusters in block 1 and block 5 corresponding to the top-5 (from left to right) joint cluster relevance for a goldfish image.  Middle: Input images and standard LRP heatmaps.
    Bottom: Heatmap explanations of the clusters for a tench image.
    The percentages indicate the joint relevance values.
    }
    \label{fig:joint_cluster}
\end{figure}

\subsection{Importance of normalization for cross-layer comparability in RNNs}
\label{sec:NormalizationProperty}

In our NRM framework, relevance is defined as a normalized signed measure, ensuring comparability across layers.
To demonstrate the importance of normalization,
we apply NRM and its unnormalized counterpart to an RNN, where interactions with input features occur at different layers.

We synthesize a toy time-series dataset with $T=100$ time steps for binary classification.  At each time step, the state takes one of three possible values represented as a one-hot vector, i.e., $\bfx_t \in \{ [1,0,0],[0,1,0],[0,0,1]\}$.  For all samples, the state is fixed to $\bfx_t = [1,0,0]$ at every time step except at $t=80$.  At  the time step $t=80$, the state is randomly set to $\bfx_{80}  = [0,1,0] $ or $ \bfx_{80} =[0,0,1]$ with equal probability, which determines the binary label as $0$ or $1$, respectively.
For reasonably well-trained models, the prediction should depend solely on $\bfx_{80}$, by construction of the dataset.  Accordingly, the explanation should reflect this by assigning  a significantly higher relevance  to $\bfx_{80}$ than the input at all other time steps.

We trained an RNN $f^{\text{RNN}}$ consisting of a 
linear recurrent layer with a hidden variable $\bfh_t \in \mathbb{R}^{16}$.
After all input states $\{\bfx_t\}$ are fed into the network, a fully connected network $f^{\text{predictor}}$ is applied to produce a final prediction. The forward computation is thus defined as 
\begin{align}
    \bfh_{0} &= \boldsymbol{0}, \qquad \bfh_t = f^{\text{RNN}}(\bfh_{t-1}, \bfx_{t-1}),&& t=1,\dots,100 ,
        \label{eq:RNNForwardOne}
\\
    \bfy &= f^{\text{predictor}}(\bfh_{100}) \in \mathbb{R}^{2}.
    \label{eq:RNNForwardThree}
\end{align}
The trained model achieves 100\% accuracy. 

Since an RNN is not a proper FFNN, we first convert it into a proper FFNN by using the procedure described in \ref{sec:GeneralNNAdaptation}.  This (virtual) conversion involves unfolding the recurrent structure and relocating  intermediate input neurons to the input layer, resulting in the proper FFNN representation shown in \Cref{fig:rnn_lrp_toy} (left).  Furthermore, since the converted proper FFNN contains multiple variables in each layer, we define the propagation matrices for each pair of variables in consecutive layers in an unnormalized form (like in Eq.~\eqref{eq:UnnormalizedPropagationMatrixGamma}).  These matrices are then combined into layer-wise propagation matrices as in Eq.~\eqref{eq:WholeUnnormalizedPropagationMatrix}, and jointly normalized to construct the consecutive conditional relevance \eqref{eq:ConsecutiveConditionalRelevanceContruction}.
A detailed procedure for defining the relevance is provided in 
\ref{sec:DetailsLinear}.

\begin{figure}[t]
\qquad \qquad
    \begin{minipage}{0.25\textwidth}
    \centering
    \includegraphics[width=\linewidth]{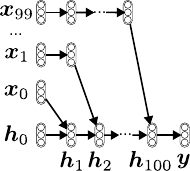}
    \end{minipage}
    \qquad \qquad
    \begin{minipage}{0.4\textwidth}
    \centering
    \includegraphics[width=\linewidth]{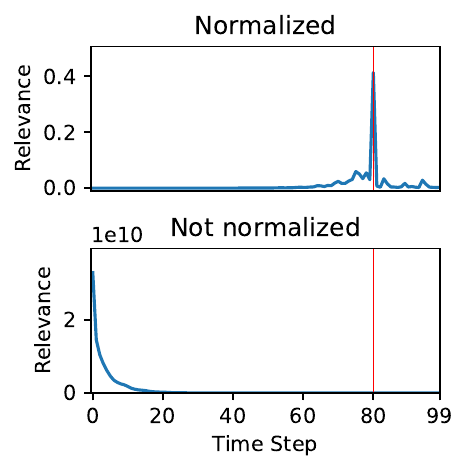}
    \end{minipage}
    \vspace{-3mm}
    \caption{ \textbf{Left}: A proper FFNN representation of RNN. \textbf{Right}: The relevance scores of input features at each time step $t$. $\bfx_{80}$ should receive the highest relevance score according to the data construction. Upper: NRM (normalized) assigns highest relevance to $\bfx_{80}$ (marked with red line). Lower: unnormalized relevance explodes in the earlier time steps. 
    }
    \label{fig:rnn_lrp_toy}
\end{figure}
To highlight the importance of normalization, we employ the LRP-$\alpha\beta$ ($\alpha=1.5, \beta=0$) rule for both properly normalized relevance and its unnormalized counterpart.
For the latter, we
directly set the consecutive conditional relevance  as $R^{\mathrm{unnorm}}(n^{(l-1)} | n^{(l)}) 
= \widetilde{T}^{(l)}_{n^{(l-1)}, n^{(l)}}$.
\Cref{fig:rnn_lrp_toy} (right) shows the relevance $R(n^{(0)} \in \mcS^{(0)}_{\bfx_t}| n^{(L)})$ of the input features  explaining the true class output neuron $n^{(L)}$ as a function of the time step $t$, where $\mcS^{(0)}_{\bfx_t}$ denotes the set of neurons representing the input feature $\bfx_t$.
We observe that 
NRM
with normalized relevance
yields a reasonable explanation, attributing  substantially higher relevance to the input feature $\bfx_{80}$ than to the others (see \Cref{fig:rnn_lrp_toy}, right top). 
In contrast, the unnormalized counterpart exhibits relevance explosion in the early time steps (see \Cref{fig:rnn_lrp_toy}, right bottom), which occurs because interactions between individual input features and the hidden variable take place at different time steps, correspondingly, at different layers in the proper FFNN representation (see \Cref{fig:rnn_lrp_toy} left).  In such cases, the total mass of relevance should remain constant over time to ensure comparability between features.  
This simple toy example illustrates the importance of normalization, which the NRM framework guarantees through its relevance construction procedure.

\section{Conclusion}
\label{sec:Conclusion}

Despite rapid advances in XAI, explaining complex models with hierarchical architectures remains a significant challenge. Higher-order explanations are likely essential for revealing the inner workings of such models.
In this paper,
we have proposed a novel Normalized Relevance Measure (NRM) framework for systematically defining the \emph{relevance of any set of neurons}, and deriving the corresponding layer-wise relevance propagation (LRP) algorithms.
Grounded in measure theory, this framework integrates principled mathematical constructions and inference techniques for defining and computing relevance, thereby guaranteeing the comparability of relevance scores across layers.

The NRM framework is applicable to a broad class of inference models. We demonstrated its versatility by applying it to an MLP trained on MNIST, a pretrained VGG16 model on ImageNet, and an RNN trained on synthetic time-series data.
Through these experiments, 
we demonstrated that joint relevance analysis within the NRM framework yields fine-grained explanations that elucidate higher-order inner workings of the network---information not accessible through standard LRP. These results highlight the potential of our approach for achieving fine-grained and faithful~XAI.

In future work, we plan to apply the NRM framework to a wider range of complex architectures. A particular focus will be on generative models, where capturing salient information flow and translating it into human-interpretable explanations are especially challenging.
We envision that the flexibility of the NRM framework will facilitate the exploration of diverse attribution strategies for internal model components, thereby helping to identify effective ways of extracting meaningful information.
Such algorithms would substantially enhance the usability of the NRM framework and further advance the field of XAI.

\section*{CRediT authorship contribution statement}
Ping Xiong: Writing – review \& editing, Writing – original draft,
Visualization, Validation, Software, Methodology, Investigation, Formal analysis, Conceptualization; 
Thomas Schnake: Writing – review \& editing, Writing – original draft, Visualization, Methodology, Investigation, Formal analysis, Conceptualization; 
Gr\'{e}goire Montavon:
Writing – review \& editing, Methodology, Conceptualization;
Klaus-Robert M\"{u}ller: Writing – review \& editing, Supervision, Resources, Methodology, Funding acquisition, Conceptualization; 
Shinichi Nakajima:
Writing – review \& editing, Writing – original draft, Supervision, Project
administration, Methodology, Investigation, Formal analysis, Conceptualization.

\section*{Declaration of competing interest}
The authors declare that they have no known competing financial
interests or personal relationships that could have appeared to influence
the work reported in this paper.

\section*{Acknowledgments}
We thank Jacob Kauffmann for helpful discussions on the cluster-level joint relevance analysis. 
This work was funded by the German Ministry for Education and Research as BIFOLD - Berlin Institute for the Foundations of Learning and Data (ref. BIFOLD25B). 
Thomas Schnake is a postdoctoral fellow at the University of Toronto in the Eric and Wendy Schmidt AI in Science Postdoctoral Fellowship Program, a program of Schmidt Sciences.
Klaus-Robert M\"{u}ller was partly supported by the Institute of Information \& Communications Technology Planning \& Evaluation (IITP) grant funded by the Korea government (MSIT) (No. RS-2019-II190079, Artificial Intelligence Graduate School Program, Korea University) and grant funded by the Korea government (MSIT, No. RS-2024-00457882, AI Research Hub Project), and Hector Fellow Academy.

\bibliographystyle{abbrvnat}
\bibliography{main.bib}

\FloatBarrier
\clearpage

\appendix
\setcounter{figure}{0}
\renewcommand{\thetheorem}{\Alph{section}.\arabic{theorem}}

\allowdisplaybreaks

\section{Generalization to Arbitrary Architectures}
\label{sec:GeneralNNAdaptation}

In \Cref{sec:NormalizedRelevanceFramework},
we described the NRM framework for a proper FFNNs---a FFNN with no skip connection and no intermediate input nor output.
However, NRM can be applied to any network with arbitrary architecture by converting it into a proper FFNN.
This conversion is only for defining the relevance structure, and therefore can be virtual, i.e., no need to modify the network forward process.

We first provide a precise definition of proper FFNNs, which depends not only on the network architecture but also on the (unnormalized) propagation matrices, defined in \Cref{sec:WalkRelevanceSpecification}.
We refer to a neuron as a \emph{source} neuron if it emits relevance but does not receive any, and as a \emph{drain} neuron if it receives relevance but does not emit any.
The following theorem holds:
\begin{theorem}
\label{thrm:TheoremSourceDrain}
    Any $L$-layered FFNN with unnormalized propagation matrices $\{\widetilde{\bfT}^{(l)}\}_{l=1}^L$ and an unnormalized output relevance $\widetilde{\bfr}^{(\ol)}$ satisfying positive-sum conditions \eqref{eq:PositiveSumConditionPropagationMatrix} and \eqref{eq:PositiveSumConditionOutputRelevance} has source neurons only in the output layer, and drain neurons only in the input layer.
\end{theorem}
\begin{proof}
Under the positive-sum conditions \eqref{eq:PositiveSumConditionPropagationMatrix} and \eqref{eq:PositiveSumConditionOutputRelevance}, the joint normalized relevance can be defined and decomposed as Eq.~\eqref{eq:JointRelevanceDecomposition}.
Since all consecutive conditional relevances $\{\Rel(n^{(l-1)} | n^{(l)})\}_{l=1}^L$ are normalized for all condition neurons $n^{(l)} \in \{1, \ldots, N^{(l)}\}$,
exactly the same amount of relevance received in any neuron in the $l$-th layer is emitted to neurons in the $(l-1)$-th layer for $l=1, \ldots, L$.   This means that no source nor drain exists in the intermediate layers.  Since relevance is back-propagated, only the output and input layers may have source and drain neurons, respectively.
\end{proof}

The contraposition of 
\Cref{thrm:TheoremSourceDrain}
implies that we cannot design the propagation matrices 
for FFNNs with source or drain neurons in intermediate layers.  Also skip connections formally breaks the Markov property.  Accordingly, we define the proper FFNNs as follows:
\begin{definition} (Proper FFNNs)
\label{dfn:ProperFFNN}
    An FFNN is called proper if it has no skip connections, has source neurons only in the output layer, and has drain neurons only in the input layer.
\end{definition}
In the relevance back-propagation, the source neurons correspond to the output neurons as starting points, and the drain neurons correspond to the input neurons as the final receivers of relevance.  Therefore, Definition \ref{dfn:ProperFFNN} requires that all output and input neurons must be in the output and input layers, respectively.%
\footnote{
\Cref{dfn:ProperFFNN} allows us to ignore certain neurons---for instance, those corresponding to bias terms in multi-layer perceptrons (MLPs)---by propagating no relevance to it nor from it.  Those neurons are omitted from our notation.
}
Below, we explain how to convert an arbitrary NN into a proper FFNN without changing the forward computation (and thus the input-output relationship).

\begin{figure}
    \centering
   \includegraphics[width=\linewidth]{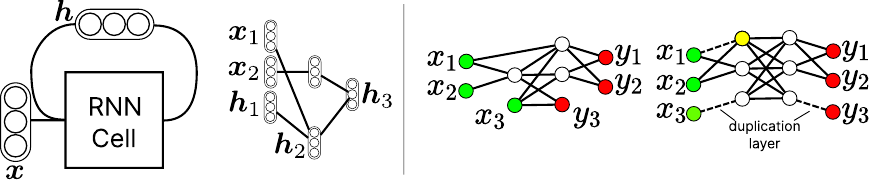}
   
    \caption{  Virtual conversion of general NNs to proper FFNNs.  
     Left: RNN with input series $\bfx_1,\bfx_2$ and hidden states $\bfh_i$ can be unfolded and transformed to a proper FFNN.  
     Right: 
    Duplicating neurons for removing skip connection ($x_1$), 
    intermediate input ($x_3$), and intermediate output ($y_3$). 
    The virtually added identity forward connections are depicted as dashed lines. 
    After adaptation, the input neurons (green) are all in the first layer, the output neurons (red) are all in the last layer, and no skip connection exists.
    }
    \label{fig:rnn_skip_to_proper_ffnn}
\end{figure}

\paragraph{Recurrent Neural Networks (RNNs)}
We can simply unfold an RNN by considering RNN cell as separate layers as shown in \Cref{fig:rnn_skip_to_proper_ffnn} (left).

\paragraph{Skip connections}
For a skip connection from layer $\underline{l}$ to layer $\overline{l}$, we simply copy the corresponding neurons at layer $\underline{l}$ to the layers $l = \underline{l}+1, \ldots, \overline{l}-1$ with the identity forward connection (see \Cref{fig:rnn_skip_to_proper_ffnn} right).

\paragraph{Intermediate inputs and outputs}
For the output neurons in the $\overline{l}$-th intermediate layer, we copy the corresponding neurons to the layers $l = \overline{l}+1, \ldots, \ol$ with the identity forward connections.
For the input neurons in the $\underline{l}$-th intermediate layer, we move it to the input layer ($l=0$), and then copy the corresponding neurons to the layers $l = 1, \ldots, \underline{l}$ with the identity forward connections
(see \Cref{fig:rnn_skip_to_proper_ffnn} right).

\section{Relevance Definitions and Relevance Laws}
\label{sec:A.RelevanceDefinisionAndLaws}
\Cref{tab:relevance_comp_descr_def} summarizes the relevance definitions and relevance laws introduced in \Cref{sec:NormalizedRelevanceFramework}.
\begin{table*}[t]
    \center
    \renewcommand{\arraystretch}{1.3}
    \caption{
    Relevance definitions.} 
    \label{tab:relevance_comp_descr_def}
    \begin{tabular}{K{.22\textwidth}K{.35\textwidth}K{.17\textwidth}|K{.21\textwidth}}
        \toprule
          \multirow{1}{*}& \multirow{1}{*}{\textbf{Relevance of neurons} }& \multicolumn{2}{l}{\textbf{Relevance of corresponding walks} } \\
          \hline
        \emph{Walk (full joint)} & $ R( \bfn )$ &  \multicolumn{2}{l}{$R^{\mathrm{Walk}}(\{\bfn \})$} \\
        \emph{Partial walk} & $R(\bfn^{(\mcL)}) =   \sum_{n^{(l)} \in \mathbb{N}^{(l)} \forall l \notin \mcL  }    R(\bfn)$ &  \multicolumn{2}{l}{$R^{\mathrm{Walk}}(\{\bfm \in \mcW: \bfm^{(\mcL)} = \bfn^{(\mcL )}\})$}  \\
        \emph{Single neuron (marginal)} & $R(n^{(l_1)})= \sum_{n^{(l)} \in \mathbb{N}^{(l)} \forall l \notin \{l_1\}  }    R(\bfn)$ &  \multicolumn{2}{l}{$R^{\mathrm{Walk}}(\{\bfm \in \mcW: m^{(l_1)} = n^{(l_1 )}\})$ }  \\
        \emph{Two neurons (joint)} & $R(n^{(l_1)}, n^{(l_2)})= \sum_{n^{(l)} \in \mathbb{N}^{(l)} \forall l \notin \{l_1, l_2\}  }    R(\bfn)$ &  \multicolumn{2}{l}{$R^{\mathrm{Walk}}(\{\bfm \in \mcW: (m^{(l_1)}, m^{(l_2)}) = (n^{(l_1 )}, n^{(l_2 )})\})$ }  \\
        \emph{Arbitrary Neuron Set} & $R(\mcS^{(\mcL)}) = \sum_{\bfn^{(\mcL)} \in \mcS^{(\mcL)}}R(\bfn^{(\mcL)})$ & \multicolumn{2}{l}{ $R^{\mathrm{Walk}}(\{\bfm \in \mcW: \bfm^{(\mcL)} \in \mcS^{(\mcL )}\})$ }  \\
        \emph{Conditional}, $\mcL_1 \cap \mcL_2 = \emptyset$ & \multicolumn{3}{l}{$ R(\bfn^{(\mcL_1)} | \bfn^{(\mcL_2)})=R(\bfn^{(\mcL_1 \cup \mcL_2)}) / R( \bfn^{(\mcL_2)})$ if $R(\bfn^{(\mcL_2)}) \neq 0$, else $0$} \\
        \bottomrule
        \emph{Complementary law} & \multicolumn{3}{l}{$R(\mcS^{(l)} ) = 1- R( \overline{\mcS}^{(l)} )$, where $ \overline{\mcS}^{(l)} = \mathbb{N}^{(l)} \setminus \mcS^{(l)}  $ } \\
        \emph{Multiplication law}  & \multicolumn{3}{l}{$ R(n^{(l_1)}, n^{(l_2)}) = R(n^{(l_1)}| \, n^{(l_2)}) R(n^{(l_2)}) $ }\\
        \emph{Addition law}  & \multicolumn{3}{l}{$R(\mcS_1^{(l)} \cup \mcS_2^{(l)}) = R( \mcS_1^{(l)} ) +R( \mcS_2^{(l)} )  - R(\mcS_1^{(l)} \cap \mcS_2^{(l)}) $ }        \\
        \bottomrule
    \end{tabular}
\end{table*}

\section{Proof of \Cref{thrm:LRPasMessagePassing}}
\label{sec:A.ProofLRPasMessagePassing}

It holds that
\begin{align}
    R(\mcS^{(\mcL)})
& = \sum_{ \bfn \in \mathbb{W} :\, n^{(l)} \in \mcS^{(l)} \forall l \in \mcL} R(\bfn),
\notag\\
&= 
\sum_{n^{(0)} \in \widetilde{\mcS}^{(0)}} 
\!\!\!\!
\cdots 
\!\!\!\!
\sum_{n^{(L)} \in \widetilde{\mcS}^{(L)}} 
\left(
  \prod_{l=1}^{L}
R(n^{(l-1)} | n^{(l)})
\right)
R(n^{(L)})
    \notag \\
&= 
 \underbrace{ 
\sum_{n^{(0)} \in \mathbb{N}^{(0)}} 
\!\!
\cdots 
\!\!\!\!\!\!
\sum_{n^{(l_1-1)} \in \mathbb{N}^{(l_1-1)}} 
\left(
  \prod_{l=1}^{l_1}
R(n^{(l-1)} | n^{(l)})
\right)
}_{=1}
\notag\\
& \hspace{5mm}
\sum_{n^{(l_1)} \in \widetilde{\mcS}^{(l_1)}} 
\Bigg(
\sum_{n^{(l_1+1)} \in \widetilde{\mcS}^{(l_1+1)}} 
\!\!\!\!\!\!
R(n^{(l_1)} | n^{(l_1+1)})
\notag\\
&\hspace{5mm} \qquad  \cdots 
 \underbrace{ 
\sum_{n^{(L)} \in \widetilde{\mcS}^{(L)}} 
R(n^{(L-1)} | n^{(L)})
\msgbel{L}_{n^{(L)}}
 }
_{= 
\msgbel{L-1}_{n^{(L-1)}}}
\Bigg)
\notag\\
&= 
\sum_{n^{(l_1)} \in \widetilde{\mcS}^{(l_1)}} 
\Bigg(
\sum_{n^{(l_1+1)} \in \widetilde{\mcS}^{(l_1+1)}} 
\!\!\!\!\!\!
R(n^{(l_1)} | n^{(l_1+1)})
\notag\\
&\hspace{5mm} \qquad  \cdots 
\!\!\!\!\!\!
 \underbrace{ 
\sum_{n^{(L-1)} \in \widetilde{\mcS}^{(L-1)}} 
\!\!\!\!\!\!
R(n^{(L-2)} | n^{(L-1)})
\msgbel{L-1}_{n^{(L-1)}}
 }
_{= 
\msgbel{L-2}_{n^{(L-2)}}}
\Bigg)
\notag\\
&=\cdots
\notag\\
&=
\!\!\!\!
\sum_{n^{(l_1)} \in \widetilde{\mcS}^{(l_1)}} 
\!\!
\Bigg(
 \underbrace{ 
\sum_{n^{(l_1+1)} \in \widetilde{\mcS}^{(l_1+1)}} 
\!\!\!\!\!\! \!\!
R(n^{(l_1)} | n^{(l_1+1)})
\msgbel{l_1+1}_{n^{(l_1+1)}}
 }
_{= 
\msgbel{l_1}_{n^{(l_1)}}}
\Bigg)
\notag\\
&=
\sum_{n^{(l_1)} \in \widetilde{\mcS}^{(l_1)}}
\msgbel{l_1}_{n^{(l_1)}},
\notag
\end{align}
where we used the normalization property of the conditional relevance $R(n^{(l)} | n^{(l+1)})$ in the third equation,
and recursively replace the last summation in the parenthesis with the corresponding message in the last equations.
\QED

\section{Proof of \Cref{thrm:EquivalenceToOriginalWalkRelevance}}
\label{sec:A.ProofWalkRelevanceRelation}

If we define the full joint relevance by the normalized version \eqref{eq:A.JointRelevanceAndWalkRelevance} of the original walk relevance definition \eqref{eq:OriginalWalkRelevance},
the conditional relevance of the neuron $n^{(l-1)}$ (for $l = 1, \ldots, L$) on the subsequent layers can be written as
\begin{align}
  & R(n^{(l-1)} | n^{(l)}, \ldots, n^{(L)})
  \notag\\
  &\hspace{3mm}= 
   \sum_{n^{(0)}} \cdots \sum_{n^{(l-2)}}
  \left(\prod_{l'=1}^{L}   \widetilde{T}^{(l')}_{n^{(l'-1)}, n^{(l')}} \right)  \widetilde{r}^{(\ol)}_{n^{(\ol)}}
  \notag\\
  &\hspace{6mm} \Bigg\{
    \sum_{n^{(0)}} \cdots \sum_{n^{(l-2)}} \sum_{n'^{(l-1)}}
  \left(\prod_{l'=1}^{l-2}  \widetilde{T}^{(l')}_{n^{(l'-1)}, n^{(l')}} \right) 
  \widetilde{T}^{(l-1)}_{n^{(l-2)}, n'^{(l-1)}} 
  \notag\\
  & \hspace{20mm}
  \widetilde{T}^{(l)}_{n'^{(l-1)}, n^{(l)}} 
    \left(\prod_{l'=l+1}^{L}  \widetilde{T}^{(l')}_{n^{(l'-1)}, n^{(l')}} \right) 
\widetilde{r}^{(\ol)}_{n^{(\ol)}}  
    \Bigg\}^{-1}
  \notag\\
  &\hspace{3mm}= \textstyle
  \frac{
  \left\{ \prod_{l'=1}^{l-1}  \fconst{l'}  \right\}  
  \widetilde{T}^{(l)}_{n^{(l-1)}, n^{(l)}}\widetilde{r}^{(\ol)}_{n^{(\ol)}}
  }
  {
  \left\{ \prod_{l'=1}^{l-1}  \fconst{l'}  \right\}  
   \sum_{n'^{(l-1)}}
  \widetilde{T}^{(l)}_{n'^{(l-1)}, n^{(l)}}\widetilde{r}^{(\ol)}_{n^{(\ol)}}
  }\notag\\
  &\hspace{3mm}= \textstyle\frac{   \widetilde{T}^{(l)}_{n^{(l-1)}, n^{(l)}}  }{  \sum_{n'^{(l-1)}}    \widetilde{T}^{(l)}_{n'^{(l-1)}, n^{(l)}}   }
 = R(n^{(l-1)} | n^{(l)}),
 \label{eq:A.MarkovProperty_1}
  \end{align}
where
\begin{align}
    C^{(l')} =\textstyle \sum_{n^{(l'-1)}}
\widetilde{T}^{(l')}_{n^{(l'-1)}, n^{(l')}}
\notag
\end{align}
 does not depend on $n^{(l')}$ under the constant-sum condition \eqref{eq:A.ConstantSumConditionPropagationMatrix}.
The Markov property, shown as Eq.~\eqref{eq:A.MarkovProperty_1},
allows us to decompose the joint relevance as 
  \begin{align}
 \Rel(n^{(\il)}, \ldots, n^{(\ol)}) & =\textstyle \left(\prod_{l=1}^{L}  \Rel(n^{(l-1)} | n^{(L)}) \right)   R(n^{(\ol)}) ,
 \label{eq:A.DefJointDecompositionProbabilityForm_1}
 \end{align}
where
 \begin{align}
 R(n^{(\ol)})
 &=
 \sum_{n^{(0)}} \cdots \sum_{n^{(L-1)}}
\left(\prod_{l'=1}^{L}   \widetilde{T}^{(l')}_{n^{(l'-1)}, n^{(l')}} \right)  \widetilde{r}^{(\ol)}_{n^{(\ol)}}
\notag\\
 & \hspace{3mm}
\Bigg\{
\sum_{n^{(0)}} \cdots \sum_{n^{(L-1)}} \sum_{n'^{(L)}}
\left(\prod_{l'=1}^{L}   \widetilde{T}^{(l')}_{n^{(l'-1)}, n^{(l')}} \right)  
\notag\\
 & \hspace{20mm}
\widetilde{T}^{(L)}_{n^{(L-1)}, n'^{(L)}} 
\widetilde{r}^{(\ol)}_{n'^{(\ol)}}
\Bigg\}^{-1}
\notag\\
  &= \textstyle
  \frac{
  \left\{ \prod_{l'=1}^{L}  \fconst{l'}  \right\}  
\widetilde{r}^{(\ol)}_{n^{(\ol)}}
  }
  {
  \left\{ \prod_{l'=1}^{L}  \fconst{l'}  \right\}  
   \sum_{n'^{(L)}}
\widetilde{r}^{(\ol)}_{n'^{(\ol)}}
  }
   = 
  \frac{
\widetilde{r}^{(\ol)}_{n^{(\ol)}}
  }
  {
   \sum_{n'^{(L)}}
\widetilde{r}^{(\ol)}_{n'^{(\ol)}}
  }.
 \label{eq:A.RelationNormalizedCase}
 \end{align}
 Eqs.\eqref{eq:A.MarkovProperty_1}--\eqref{eq:A.RelationNormalizedCase} exactly match the relevance definitions in Eqs.\eqref{eq:ConsecutiveConditionalRelevanceContruction}, \eqref{eq:JointRelevanceDecomposition}, and \eqref{eq:OutputRelevanceContruction}, respectively, in the NRM framework, 
 which completes the proof.
 \QED

\section{Original LRP Algorithm and Common LRP Rules}
 \label{sec:A.CommonLRPRules}

In the original standard LRP \citep{bach2015pixel,DBLP:journals/pieee/SamekMLAM21}, the \emph{initial} unnormalized output relevance $\widetilde{\bfr}^{(\ol)} \in \mathbb{R}^{\numneu{\ol}}$ is set according to the output $\bff(\bfx)$ of the FFNN to be explained, e.g., as $\widetilde{\bfr}^{(\ol)}  = 
f_{n^{(\ol)}}(\bfx)\bfone_{n^{(\ol)}} $
or $\widetilde{\bfr}^{(\ol)}  = \bfone_{n^{(\ol)}} $,
where $\bfone_{n}  \in \mathbb{R}^N$ for $n \in \{ 1, \ldots, N\}$ is the one-hot-vector, i.e., the $n$-th entry is one and the other entries are zero.  
Then, the output relevance $\widetilde{\bfr}^{(\ol)}$ is propagated as
\begin{align}
  \prelvec{l-1} & = \widetilde{\bfT}^{(l)}   \prelvec{l} \qquad  \mbox { for } \qquad  l = L, \ldots, 1, 
    \label{eq:A.LRP_rule}
    \end{align}
where
    $\prelvec{l} \in \mathbb{R}^{\numneu{l}}$ is the unnormalized  relevance at the $l$-th layer.
The unnormalized propagation matrices $\{\widetilde{\bfT}^{(l)}\}_{l=1}^L$ are typically set in (approximately) normalized forms, e.g.,
\begin{align}
& \varepsilon \mbox{ rule: }
&
 \widetilde{T}^{(l)}_{n, n'}& = \textstyle
\frac{ h^{(l-1)}_{n} W^{(l)}_{n, n'} }{\varepsilon + \sum_{n''} {h^{(l-1)}_{n''} W^{(l)}_{n'', n' } }},
\label{eq:A.UnnormalizedPropagationMatrixEpsilon} \\
& \alpha\beta \mbox{ rule: }
&
 \widetilde{T}^{(l)}_{n, n'}& = \textstyle
 \alpha
\frac{ \max(0, h^{(l-1)}_{n} W^{(l)}_{n, n'} )}{ \sum_{n''} \max(0, h^{(l-1)}_{n''} W^{(l)}_{n'', n' } )}
- \beta
\frac{ \max(0, - h^{(l-1)}_{n} W^{(l)}_{n, n'} )}{ \sum_{n''} \max(0,- h^{(l-1)}_{n''} W^{(l)}_{n'', n' } )},
\label{eq:A.UnnormalizedPropagationMatrixAlphaBeta}\\
& \gamma \mbox{ rule: }
&
 \widetilde{T}^{(l)}_{n, n'}
 &= 
 \textstyle
\frac{ h^{(l-1)}_{n} W^{(l)\uparrow}_{n, n'} }{\sum_{n''} {h^{(l-1)}_{n''} W^{(l)\uparrow}_{n'', n' } }},
\label{eq:A.UnnormalizedPropagationMatrixGammaNormalized} 
\end{align}
where $ W^{(l)\uparrow}_{n, n'} = W^{(l)}_{n, n'} + \gamma \max(0, W^{(l)}_{n, n'})$.

Most of the existing LRP rules, including the LRP-$\gamma$ rule \eqref{eq:A.UnnormalizedPropagationMatrixGammaNormalized}, satisfy the normalization condition \eqref{eq:A.NormalizedConditionPropagationMatrix} for the propagation matrices.
Furtheremore, although the output relevance is commonly not normalized, its normalization only affects the absolute scale, which is not essential in typical use cases of LRP, e.g., for depicting heat-maps and ranking relevant features.  Therefore, we can additionally assume that the output relevance is normalized so that Eq.~\eqref{eq:A.NormalizedConditionOutputRelevance} holds.
In this case, 
the consecutive conditional relevance, the output relevance, and the full joint relevance in our NRM framework coincide with the propagation matrix, the output relevance, and the original walk relevance, respectively (see Eq.
\eqref{eq:A.EquivalanceNormalizedWalkRelevance}).

 Some LRP rules, including the LRP-$\alpha \beta$ rule \eqref{eq:A.UnnormalizedPropagationMatrixAlphaBeta} for $\beta \ne \alpha - 1$, do not satisfy the normalization condition \eqref{eq:A.NormalizedConditionPropagationMatrix}, but still satisfy the \emph{constant-sum} condition \eqref{eq:A.ConstantSumConditionPropagationMatrix}, which requires that all columns of the propagation matrix $\widetilde{\bfT}^{(l)}$ sum up to the same positive constant $C^{(l)}$.  Still in this case, the full joint relevance in the NRM framework matches the original definition of the walk relevance up to a scaling factor, i.e., the joint relevance in the NRM framework corresponds to the normalized version of the original walk relevance \citep{schnake2020higher} (see Eq.~\eqref{eq:A.JointRelevanceAndWalkRelevance}).

A few LRP rules, including the LRP-$\varepsilon$ rule \eqref{eq:A.UnnormalizedPropagationMatrixEpsilon}, do not even satisfy  the constant-sum condition \eqref{eq:A.ConstantSumConditionPropagationMatrix}.  In such cases, we cannot define the joint relevance by normalizing the original definition of the walk relevance as in Eq.~\eqref{eq:A.JointRelevanceAndWalkRelevance}, because this breaks the backward Markov property---the basic property that justifies the layer-wise back-propagation of the LRP algorithms.
This is the reason why we defined the relevance function on the sets of walks that correspond to the consecutive conditional relevances \eqref{eq:ConsecutiveConditionalRelevanceContruction} and the output marginal relevance \eqref{eq:OutputRelevanceContruction} in \Cref{sec:WalkRelevanceSpecification}, instead of defining the relevance function on the set of all single walks.
With our walk relevance specification  procedure in \Cref{sec:WalkRelevanceSpecification}, we can consistently define the relevance measure space for most general propagation matrices and output relevance vectors that satisfy the weakest requirements---the postive-sum conditions \eqref{eq:PositiveSumConditionPropagationMatrix} and \eqref{eq:PositiveSumConditionOutputRelevance}.

\section{Recasting Existing LRP Algorithms within NRM Framework}
\label{sec:A.RalationToLRPAlgorithms}

LPR algorithms are typically defined as back-propagation algorithms with additional procedures, e.g., masking neurons in some layers and stopping propgation at an intermediate layer.  By recasting existing LRP algorithms within the NRM framework, we here identify the target relevance quantities they compute.
In most common LRP rules, the propagation matrices are (approximately) normalized, while the output relevance is not.  However, the scaling of the output relevance affects only the absolute scale, which is not important when, e.g., depicting heat-maps and ranking the relevant features. Accordingly, we here assume the normalization conditions \eqref{eq:A.NormalizedConditionPropagationMatrix} and \eqref{eq:A.NormalizedConditionOutputRelevance}.

\begin{figure}[!t]
    \centering
   \includegraphics[width=0.8\linewidth, trim={0 0 287 0}, clip]{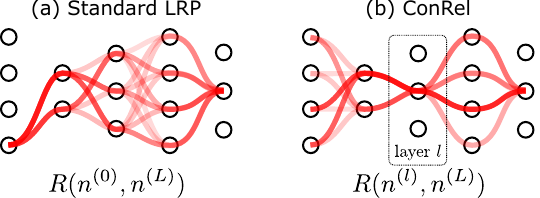}
   \includegraphics[width=0.8\linewidth, trim={287 0 0 0}, clip]{img/RelationExisitingLRP.pdf}
    \caption{
    Existing LRP methods visualized as the relevance of sets of walks to be computed.  (a) The relevance of an input neuron $n^{(\il)}$ for explaining an output neuron $n^{(\ol)}$.  (b) ConRel, which is the relevance of an intermediate \emph{concept} neuron $n^{(l)}$ for explaining an output neuron $n^{(\ol)}$.  This is also the objective to be maximized for RelMax.  (c) CRP, which is the relevance of an input neuron $n^{(\il)}$ for explaining an intermediate neuron $n^{(l)}$ and the output neuron $n^{(\ol)}$. (d) Subgraph relevance.  The square indicates the GNN part.
    }
    \label{fig:A.RelationExistingLRP}
\end{figure}

\paragraph{Standard LRP (to explain output with input relevances)}

In the standard LRP algorithm \citep{bach2015pixel}, to compute the relevances of the input neurons $n^{(\il)} \in \{1, \ldots, N^{(\il)}\}$ with respect to a chosen output neuron $n^{(\ol)}$, 
the propagated relevance by the LRP algorithm \eqref{eq:A.LRP_rule} corresponds to the sum-product message passing \eqref{eq:A.LRPasMP}, and it holds that
$\widetilde{\bfr}^{(l)}
 \propto
\msgb{l}$.
Therefore, if the relevance at the output layer is initialized as $\widetilde{\bfr}^{(\ol)} = \bfone_{n^{(\ol)}}$,
the final quantity that the LRP algorithm computes corresponds to 
the conditional relevance, i.e., for $n^{(\il)} \in \{1, \ldots, N^{(\il)}\}$,
\begin{align}
\mathrm{LRP}_{n^{(\il)}| n^{(\ol)}}^{(\mathrm{input} \leftarrow \mathrm{output})}
&=
\widetilde{r}_{n^{(\il)}}^{(\il)}
 =
R(n^{(\il)}|  {n}^{(\ol)}).
\label{eq:A.PropagatedRelevanceMarginalRelevanceInitialOne}
\end{align}
Alternatively, if the output relevance is initialized as $\widetilde{\bfr}^{(\ol)}  
= R(n^{(\ol)}) \bfone_{n^{(\ol)}}$,
the LRP algorithm computes
the joint relevance, i.e.,
\begin{align}
\mathrm{LRP}_{n^{(\il)}, n^{(\ol)}}^{(\mathrm{input} \leftarrow \mathrm{output})}
&=
\widetilde{r}_{n^{(\il)}}^{(\il)}
=
R(n^{(\il)}|  {n}^{(\ol)})
R(n^{(\ol)}) \notag \\
&=
R(n^{(\il)},  {n}^{(\ol)})  .
\label{eq:A.PropagatedRelevanceMarginalRelevanceInitialOutput}
\end{align}
\Cref{fig:A.RelationExistingLRP} (a) illustrates the standard LRP \eqref{eq:A.PropagatedRelevanceMarginalRelevanceInitialOutput} as the relevance of a set of walks.

\paragraph{CRP and RelMax} \citet{ConveptRelevancePropagation} proposed Concept Relevance Propagation (CRP) and Relevance Maximization (RelMax) for assessing the importance of intermediate neurons.  The proposed procedure consists of the following steps:
\begin{enumerate}
\item 
\label{stp:CRPRelMax.FindIntNeuron}
For a target output neuron $n^{(\ol)}$, i.e., the target class in classification, identify \emph{concept neurons}, i.e., relevant intermediate neurons $\{n^{(l)}\}$ in some layer $l \in \{1, \ldots, L-1\}$.
\item 
\label{stp:CRPRelMax.RelMax}
For each concept neuron $n^{(l)}$, find input samples that maximize the concept relevance.  This process is called Relevance Maximization (RelMax).

\item 
\label{stp:CRPRelMax.CRP}
With respect to each concept neuron $n^{(l)}$ and an output neuron $n^{(\ol)}$, compute the input relevances.  This algorithm is called CRP.
\end{enumerate}
The concept relevance (ConRel) computed in Step \ref{stp:CRPRelMax.FindIntNeuron} corresponds to the joint relevance
\begin{align}
\mathrm{ConRel}_{n^{(l)}, n^{(\ol)}}
&=
R(n^{(l)}  , n^{(\ol)}).
\label{eq:A.ConRel}
\end{align}
The quantity \eqref{eq:A.ConRel} is also the objective function to be maximized with respect to the input sample $\bfx \in \mcX$ in Step \ref{stp:CRPRelMax.RelMax}, where $\mcX$ is a set of samples.  Namely, RelMax performs
\begin{align}
\widehat{\bfx}^{\mathrm{RelMax}}
= \argmax_{\bfx \in \mcX}
R_{\bfx}(n^{(l)} , n^{(\ol)})
\label{eq:A.RelMax}
\end{align}
to identify the input for which the relevance of the concept $ n^{(l)}$ for the output $n^{(\ol)}$ is highest.
The conditional relevance computed in Step \ref{stp:CRPRelMax.CRP} corresponds to
\begin{align}
\mathrm{CRP}_{n^{(\il)} , n^{(l)}, n^{(\ol)}}
&=
R(n^{(\il)}, n^{(l)} ,  n^{(\ol)}).
\label{eq:A.CRP}
\end{align}
\Cref{fig:A.RelationExistingLRP} (b) and (c) illustrate ConRel \eqref{eq:A.ConRel} (or the objective of RelMax)  and CRP \eqref{eq:A.CRP}, respectively.

\paragraph{Relevance of set of neurons}
We are often interested in the relevance of a set of neurons, e.g., the relevance of a patch or superpixel---specified by a set $\mcS^{(\il)}$ of input neurons---in the input image, and the relevance of a channel including all spatial locations---specified by a set $\mcS^{(l)}$ of intermediate neurons---of a convolutional layer output.  In such cases, the standard LRP, ConRel, and CRP, respectively, correspond to
\begin{align}
\mathrm{LRP}_{\mcS^{(\il)}, n^{(\ol)}}^{(\mathrm{input} \leftarrow \mathrm{output})}
&=
R(n^{(\il)}\in \mcS^{(\il)},  {n}^{(\ol)})  ,
\notag\\
\mathrm{ConRel}_{\mcS^{(l)}, n^{(\ol)}}
&=
R(n^{(l)} \in \mcS^{(l)}, n^{(\ol)}),
\notag\\
\mathrm{CRP}_{\mcS^{(\il)} , \mcS^{(l)}, n^{(\ol)}}
&=
R(n^{(\il)} \in \mcS^{(\il)}, n^{(l)} \in \mcS^{(l)},  n^{(\ol)}).
\notag
\end{align}

\paragraph{Subgraph relevance:}
To explain GNNs, \citet{schnake2020higher} proposed the subgraph relevance as the sum of relevances of the walks that stay in the node embedding neurons in a given subgraph.
This corresponds to the substructure relevance \eqref{eq:A.SubstructureRelevance} conditioned on the output neuron $n^{(\ol)}$:
\begin{align}
\mathrm{SubgLRP}_{\mcS^{(\underline{l}:\overline{l})}| n^{(\ol)} }
&=
    R(n^{(\underline{l}:\overline{l})} \in 
 \mcS^{(\underline{l}:\overline{l})} | n^{(\ol)})
 \label{eq:A.SubgraphRelevance} 
\end{align}
for the
substructure $\mcS^{(\underline{l}:\overline{l})} = (\mathcal{S}, \ldots, \mathcal{S})$, where $\mathcal{S}$ corresponds to the node embedding neurons included in a given subgraph within the GNN layers $l = \underline{l}, \ldots, \overline{l}$.  
\Cref{fig:A.RelationExistingLRP} (d) illustrates the subgraph relevance \eqref{eq:A.SubgraphRelevance}.

\begin{table*}[t]
\caption{NRM expression of the target relevance quantities that existing LRP algorithms compute. 
}
\vspace{-4mm}
\center
\begin{tabular}{K{.42\textwidth}K{.52\textwidth}}
\toprule
LRP Algorithm
& Target relevance quantity \\
\midrule
Standard LRP & $ R(n^{(\il)},  {n}^{(\ol)})$ \\
ConRel & $ R(n^{(l)}, n^{(\ol)})$ \\
RelMax & $\argmax_{\bfx \in \mcX} R_{\bfx}(n^{(l)}, n^{(\ol)})$\\
CRP & $R(n^{(\il)}, n^{(l)},  n^{(\ol)})$\\
\midrule
Standard LRP for input neuron set & $R(n^{(\il)}\in \mcS^{(\il)},  {n}^{(\ol)})$\\
ConRel for concept neuron set & $R(n^{(l)} \in \mcS^{(l)}, n^{(\ol)})$\\
RelMax for concept neuron set  & $\argmax_{\bfx \in \mcX} R_{\bfx}(n^{(l)} \in \mcS^{(l)}, n^{(\ol)})$\\
CRP for input and concept neuron sets & $ R(n^{(\il)} \in \mcS^{(\il)}, n^{(l)} \in \mcS^{(l)},  n^{(\ol)})$\\
\midrule
Subgraph LRP & $R(n^{(\underline{l}:\overline{l})} \in \mcS^{(\underline{l}:\overline{l})} | n^{(\ol)})$\\
SymbXAI for set $\mcS$  & $  1 - R(n^{(\underline{l}:\overline{l})} \in \overline{\mathcal{S}}^{(\underline{l}:\overline{l})}| n^{(\ol)})$\\
 SymbXAI for formula $\mcS_1 \land \mcS_2$ 
 &
$1 
        +R(n^{(\underline{l}:\overline{l})} \in \overline{\mathcal{S}_1 \cup \mathcal{S}_2}^{(\underline{l}:\overline{l})}| n^{(\ol)}) 
        \qquad  \qquad $
   ${ }\;\; - R(n^{(\underline{l}:\overline{l})} \in \overline{\mathcal{S}_1}^{(\underline{l}:\overline{l})}| n^{(\ol)}) 
    - R(n^{(\underline{l}:\overline{l})} \in \overline{\mathcal{S}_2}^{(\underline{l}:\overline{l})}| n^{(\ol)}) 
    $
\\
SymbXAI for formula $ \neg \mcS$  & $R(n^{(\underline{l}:\overline{l})} \in \overline{\mathcal{S}}^{(\underline{l}:\overline{l})}| n^{(\ol)})$\\
\bottomrule
\end{tabular}
\label{table:LRPinNRFramework}
\end{table*}

\paragraph{Symbolic XAI (SymbXAI):}  
\citet{SCHNAKE2025} extended the attribution of relevance from a set of single features to more complex relationships between features, which are expressed as \emph{logical formulas of subgraphs}. The framework is called Symbolic XAI (SymbXAI).
When SymbXAI is applied to a GNN using LRP as the base explanation method, the subgraph relevance  \eqref{eq:A.SubgraphRelevance} serves as the basic quantity.  
For the set $\mathcal{S}$ of node embedding neurons of a given subgraph and its complement $\overline{\mcS}$,
the SymbXAI relevance of the subset $\mathcal{S}$ is defined as  
\begin{align}
\mathrm{SymbXAI}_{\mcS | n^{(\ol)} }
&=
   1 - R(n^{(\underline{l}:\overline{l})} \in \overline{\mathcal{S}}^{(\underline{l}:\overline{l})}| n^{(\ol)}),
 \label{eq:A.fo_symbxai} 
\end{align}
where 
$\overline{\mathcal{S}}^{(\underline{l}:\overline{l})} = (\overline{\mathcal{S}},  \dots, \overline{\mathcal{S}})$.
Since the second term is the subgraph relevance \eqref{eq:A.SubgraphRelevance} for the complement $\overline{\mcS}$,
SymbXAI \eqref{eq:A.fo_symbxai} quantifies the relevance of all walks that visit at least one neuron inside the subset $\mcS$.
SymbXAI also quantifies the relevance of subsets connected by the logical conjunction, i.e., the \textrm{AND} operator denoted by~$\wedge$,
as well as the logical negation, i.e., the \textrm{NOT} operator denoted by~$\neg$.
The conjunction should follow the \emph{inclusion-exclusion principle}---a basic concept in the measure theory. Namely, 
the SymbXAI relevance of the conjunction $\mathcal{S}_1 \wedge \mathcal{S}_2$ for two sets $\mathcal{S}_1$ and $\mathcal{S}_2$ is defined as  
\begin{align}
    \mathrm{SymbXAI}_{\mcS_1 \land \mcS_2| n^{(\ol)} } 
    &\!=\! \mathrm{SymbXAI}_{\mcS_1| n^{(\ol)} } \!+ \mathrm{SymbXAI}_{ \mcS_2| n^{(\ol)} } \notag \\
     & \qquad - \mathrm{SymbXAI}_{\mcS_1 \cup \mcS_2| n^{(\ol)} }
    \notag\\
    &= 1 
        +R(n^{(\underline{l}:\overline{l})} \in \overline{\mathcal{S}_1 \cup \mathcal{S}_2}^{(\underline{l}:\overline{l})}| n^{(\ol)}) 
    \notag\\ 
    & \qquad
    - R(n^{(\underline{l}:\overline{l})} \in \overline{\mathcal{S}_1}^{(\underline{l}:\overline{l})}| n^{(\ol)}) \notag \\
    & \qquad - R(n^{(\underline{l}:\overline{l})} \in \overline{\mathcal{S}_2}^{(\underline{l}:\overline{l})}| n^{(\ol)})  .
    \label{eq:A.and_symbxai}
\end{align}
Eq.~\eqref{eq:A.and_symbxai} quantifies the relevance of all walks that visit at least one neuron inside the subset $\mcS_1$ and also at least one neuron inside the subset $\mcS_2$. 
The SymbXAI relevance for 
the negation---for assessing the relevance of \emph{feature absence}---%
is defined as  
\begin{align}
    \mathrm{SymbXAI}_{ \neg \mcS| n^{(\ol)} } 
    &=  1 -   \mathrm{SymbXAI}_{ \mcS| n^{(\ol)} } 
    \notag\\
    &= R(n^{(\underline{l}:\overline{l})} \in \overline{\mathcal{S}}^{(\underline{l}:\overline{l})}| n^{(\ol)}),
    \label{eq:A.not_symbxai}
\end{align}
which quantifies the relevance of all walks that never visit any neuron inside the subset $\mcS$.
The SymbXAI relevance definitions \eqref{eq:A.and_symbxai} and \eqref{eq:A.not_symbxai} 
determine the relevance 
of any logical formula 
due to the functional completeness of the \textrm{AND} and \textrm{NOT} operators---%
any logical formula involving also other logical connectives, e.g., \textrm{OR} and \textrm{XOR},
can be transformed to a formula only with \textrm{AND} and \textrm{NOT} operators.  
The SymbXAI formulation is consistent with 
the \emph{Harsanyi dividends} \citep{Harsanyi+1959+325+356}---a basic concept in cooperative game theory---and allows for incorporating desirable properties into the explanation method.

\medskip

\Cref{table:LRPinNRFramework} summarizes the target quantities that existing LRP algorithms compute.
Note the significant difference between the subgraph relevance \citep{schnake2020higher}
and the SymbXAI relevance \citep{SCHNAKE2025} for a given subgraph $\mcS$.  This difference can be  intuitively understood by expressing the target quantities as conditional substructure relevances \eqref{eq:A.SubgraphRelevance} and \eqref{eq:A.fo_symbxai}, respectively, within our NRM framework.

\section{Details of Additive Relevance}
\label{sec:A.DetailsAdditiveRelevance}

\subsection{Proof of \Cref{thrm:Properties.AdditiveRelevance}}

It holds that
\begin{align}
    R^{\mathrm{Add}}(\emptyset) &=\sum_{\bfn \in \mathbb{W}} R^{\mathrm{Walk}}(\bfn) \phi(\bfn, \emptyset) \notag \\
    &= \sum_{\bfn \in \mathbb{W}} R^{\mathrm{Walk}}(\bfn)\frac{\sum_{n \in \bfn} \mathbbm{1} (n \in \emptyset )}{|\bfn|} \notag \\
    &= 0,
    \notag \\
    R^{\mathrm{Add}}(\mathbb{S}) &=\sum_{\bfn \in \mathbb{W}} R^{\mathrm{Walk}}(\bfn) \phi(\bfn, \mathbb{S}) \notag \\
    &= \sum_{\bfn \in \mathbb{W}} R^{\mathrm{Walk}}(\bfn)\frac{\sum_{n \in \bfn} \mathbbm{1} (n \in \mathbb{S} )}{|\bfn|} \notag \\
    &= \sum_{\bfn \in \mathbb{W}} R^{\mathrm{Walk}}(\bfn) \notag \\
    &= 1,
    \notag \\
    R^{\mathrm{Add}}(\mcS_1 \cup \mcS_2) &=\sum_{\bfn \in \mathbb{W}} R^{\mathrm{Walk}}(\bfn) \phi(\bfn, \mcS_1 \cup \mcS_2) \notag \\
    &= \sum_{\bfn \in \mathbb{W}} R^{\mathrm{Walk}}(\bfn)\frac{\sum_{n \in \bfn} \mathbbm{1} (n \in \mcS_1 \cup \mcS_2) }{|\bfn|} \notag \\
    &= \sum_{\bfn \in \mathbb{W}} R^{\mathrm{Walk}}(\bfn)\frac{\sum_{n \in \bfn} [\mathbbm{1} (n \in \mcS_1) + \mathbbm{1} (n \in \mcS_2)] }{|\bfn|} \notag \\
    &= \sum_{\bfn \in \mathbb{W}} R^{\mathrm{Walk}}(\bfn)\frac{\sum_{n \in \bfn} \mathbbm{1} (n \in \mcS_1)  }{|\bfn|} 
     + \sum_{\bfn \in \mathbb{W}} R^{\mathrm{Walk}}(\bfn)\frac{\sum_{n \in \bfn} \mathbbm{1} (n \in \mcS_2)}{|\bfn|} \notag  \\
     &=R^{\mathrm{Add}}(\mcS_1)+R^{\mathrm{Add}}(\mcS_2).  && \qed \notag
\end{align}

\subsection{Numerical comparison with joint relevance}

\Cref{fig:comp_def_neuron_set} compares the features captured by the additive (middle column) and joint (right column) relevances of different sets of neurons (left column) for an example image input.
(see \ref{sec:A.SettingFigureAdditiveJointComparison} for the experimental set up).  
The figure demonstrates that additive explanations exhibit limited sensitivity to the choice of subgraph $\mathcal{S}$. This lack of sensitivity likely arises from the additive measure being too accommodating, as it accumulates the contribution of walks that only partially traverse the subgraph. 
In contrast, the joint relevance integrates walk relevance only when the walk remains confined within the given subgraph, achieving a more sensitive attribution model.

\begin{figure}[t]
    \centering
    \includegraphics[width=0.67\linewidth]{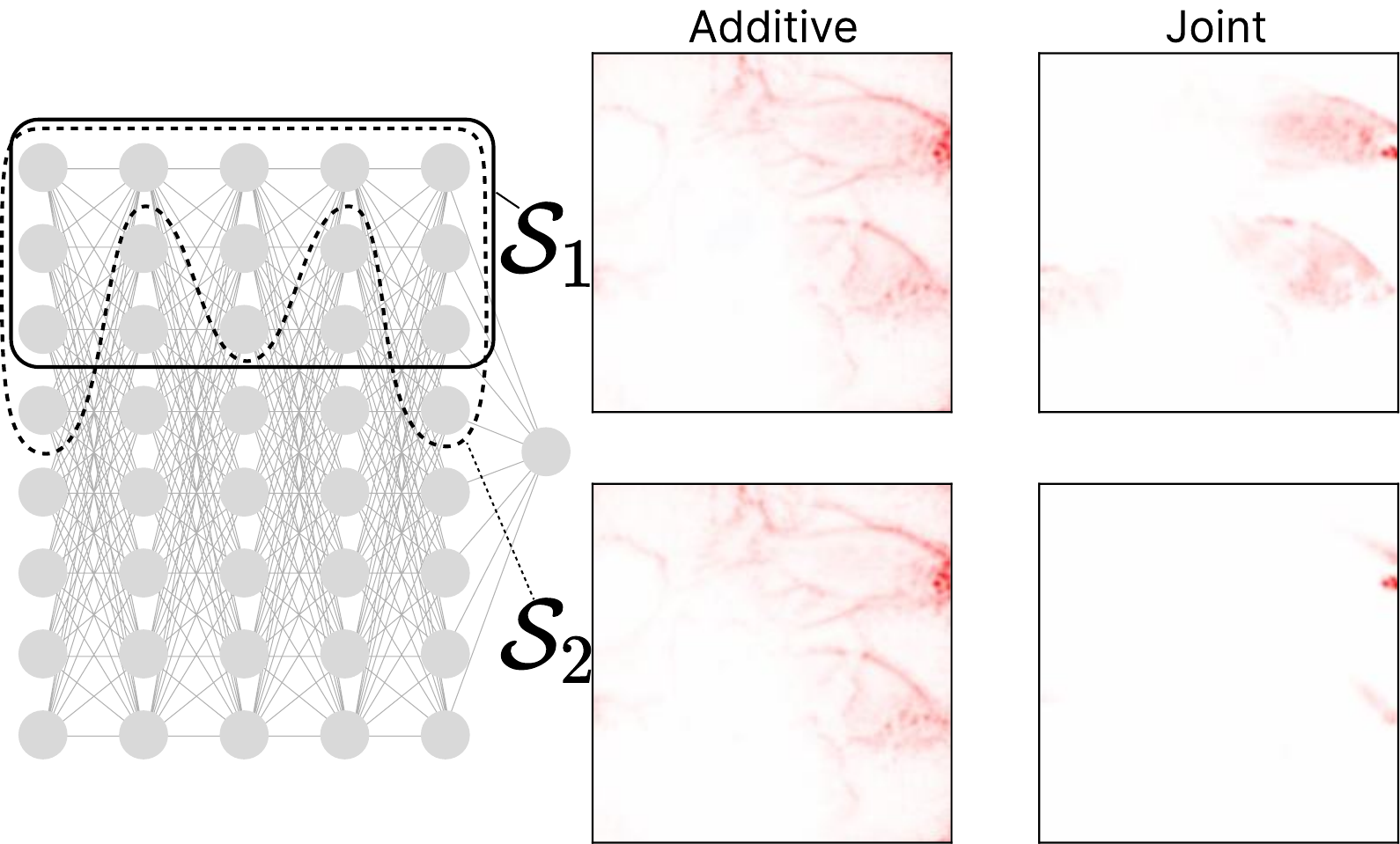}
    \vspace{-3mm}
    \caption{
    Features captured by additive (middle column) and joint (right column) relevances of  different sets of neurons (left column) over five layers 
    in VGG16.  
    }
    \label{fig:comp_def_neuron_set}
\end{figure}

\subsection{Setting for depicting \Cref{fig:comp_def_neuron_set}}
\label{sec:A.SettingFigureAdditiveJointComparison}

To compare the two definitions of neuron set relevance, we conducted the clustering experiment as in \Cref{sec:cluster_experiment}, and explained two sets of neurons, illustrated in \Cref{fig:comp_def_neuron_set} (left).
The middle and right columns
show the heatmaps produced by the additive (\Cref{def:AdditiveRelevanceDefinition}) and the joint (\Cref{dfn:RelevanceNeurons}) relevance definitions, respectively. 
The top row shows the features captured by the neuron set $\mcS_1$ concisting of the first 3 clusters in the output layer of each of the five blocks of VGG16.  The bottom row, on the other hand, shows the features captured by the neuron set $\mcS_2$ consisting of a different selection of clusters in each layer (see the illustration on the left).

\section{Procedure for Applying NRM}
\label{app:procedure_nrm}

Here we summarize the procedure for explaining a given network and input sample in the NRM framework.
\begin{enumerate}
    \item If the given network is not proper FFNN, convert it into a proper FFNN, following the procedure in \ref{sec:GeneralNNAdaptation}.
    \item Feed the input sample, and then compute the unnormalized relevance propagation matrices $\{\widetilde{\bfT}^{(l)}\}_{l=1}^L$ and output relevance $\widetilde{\bfr}^{(L)}$, 
    from which the consecutive conditional relevance \eqref{eq:ConsecutiveConditionalRelevanceContruction} and the output relevance \eqref{eq:OutputRelevanceContruction} are computed.       
    \item Choose a set of neurons $\mcS^{(\mcL)}$ to which we attribute, and the target neuron $n^{(\overline{l})}$ to be explained for $\overline{l} > \max(\mcL)$ (typically $\overline{l}  = L$).
    \item \label{stp:LRPComputation}
    Compute the relevance $R(\mcS^{(\mcL)}, n^{(\overline{l})})$ or $R(\mcS^{(\mcL)} | n^{(\overline{l})})$ by the sum-product message passing algorithm (\Cref{thrm:LRPasMessagePassing}).
    \item If we search for most relevance information flow, iterate Step \ref{stp:LRPComputation} for different set $\mcS^{(\mcL)}$ of neurons.
\end{enumerate}
With this procedure, one can analyze the decision-making mechanism of any type of NN with arbitrarily high-order explanation.

\section{Evaluation Methods}
\label{sec:eval_methods}

This appendix provides details of our evaluation metrics.

\subsection{Joint contribution}

The joint relevance $R(\mcS^{(\mcL)}| n^{(L)})$ under the NRM framework attributes the information flow passing through a specified set of neurons $\mcS^{(\mcL)}$ across layers $\mcL \subseteq \mathbb{L}$.
Here, we provide details of the \emph{joint contribution} \eqref{eq:M.inclusion_exclusion_generalNeuronsLevel}, an intervention-based metric for evaluating the quality of joint relevance.

One can evaluate the contribution of the information flow passing through a single neuron $n$ by 
\begin{align}
A(\{n\}| n^{(L)}) = f_{n^{(L)}}^{\mathrm{neu}} (\bfx) - f_{n^{(L)}}^{\mathrm{neu}} (\bfx; \backslash \{n\}),
\notag
\end{align}
which evaluates the difference in network output with and without the removal of $n$.
Similarly, we can evaluate the contribution of the information flow passing through \emph{at least one neuron} in the set $\mcS$ by 
\begin{align}
A(\mcS| n^{(L)}) = f_{n^{(L)}}^{\mathrm{neu}} (\bfx) - f_{n^{(L)}}^{\mathrm{neu}} (\bfx; \backslash \mcS).
\notag
\end{align}

For evaluating the contribution of the information flow passing through 
\emph{at least one neuron in $\mcS^{{(l)}}$ for all $l \in \mcL$}---to which the joint  relevance
$R(\mcS^{(\mcL)}| n^{(L)}) = R(n^{(l_1)} \in \mcS^{(l_1)}, \ldots, n^{(l_{|\mcL|})} \in \mcS^{(l_{|\mcL|})}| n^{(L)})$ is attributed---we employ the inclusion-exclusion principle \citep{Rota1964}:
\begin{align}
       C(\mcS^{(\mcL)}| n^{(L)}) &= \sum_{ \mcB \subseteq \mcL } (-1)^{|\mcB| + 1} A(\mcS^{(\mcB)}| n^{(L)}), 
       \end{align}
       which can be rewritten as
\begin{align}
C(\mcS^{(\mcL)}| n^{(L)}) 
        &=\!\! \sum_{ \mcB \subseteq \mcL } 
        \!\!
        (-1)^{|\mcB| + 1} \left(f_{n^{(L)}}^{\mathrm{neu}} (\bfx) - f_{n^{(L)}}^{\mathrm{neu}} (\bfx; \backslash \mcS^{(\mcB)})\right) \notag\\
        &= \underbrace{\sum_{ k=0 }^{|\mcL |} \binom{|\mcL |}{k} (-1)^{k + 1}}_{=-(1-1)^{|\mcL|} = 0} f_{n^{(L)}}^{\mathrm{neu}} (\bfx)
        \notag \\
        &\qquad \qquad + \sum_{ \mcB \subseteq \mcL  } (-1)^{|\mcB|}f_{n^{(L)}}^{\mathrm{neu}} (\bfx;  \backslash \mcS^{(\mcB)}) \notag\\
        &= \sum_{ \mcB \subseteq \mcL } (-1)^{|\mcB|} f_{n^{(L)}}^{\mathrm{neu}} (\bfx; \backslash \mcS^{(\mcB)}).
        \label{eq:M.inclusion_exclusion_generalNeuronsLevel}
\end{align}
The joint contribution \eqref{eq:M.inclusion_exclusion_generalNeuronsLevel} generalizes the notion of \emph{interpretable features}, which explain the interacting contribution of two features \citep{BLUCHER2022103774,gorbani2020neuron_shap,SCHNAKE2025}.

\subsection{Evaluation metrics} 

Based on the joint contribution introduced above,
we define two evaluation metrics, Pearson correlation and the sum of top-$k$ joint contributions.
The Pearson correlation is simply the correlation between the relevance measure and the joint contribution. Higher values indicate better performance.

To calculate the sum of top-$k$ joint contribution, we first identify the top-$k$ most relevant information flows by explanation methods, and then compute the sum of the joint contributions over them:
\begin{align}
    \text{SC}_k 
    &= \sum_{i=1}^k  C((\mcS^{(\mcL)})_i| n^{(L)}) ,
\end{align}
where $(\mcS^{(\mcL)})_i$ denotes the $i$-th most relevant combination of (neuron-wise, channel-wise or cluster-wise) neuron partitions over the layers $\mcL$. Higher values indicate better performance.

\begin{figure*}[t]
    \centering
    \centering
     \begin{subfigure}[b]{\textwidth}
         \centering
         \includegraphics[width=\linewidth]{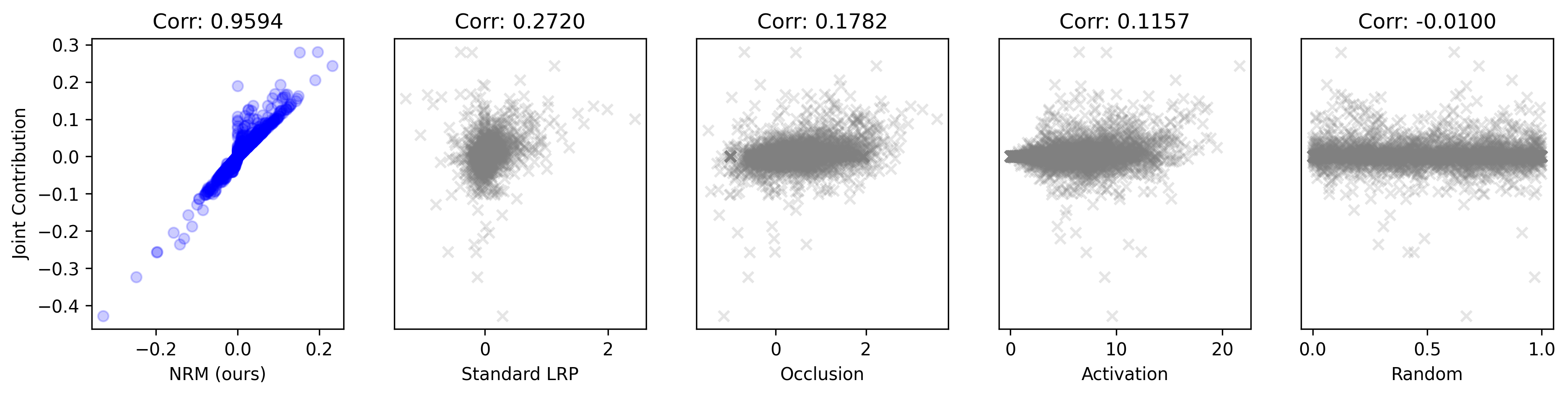}
         \caption{
         Second-order joint relevance $R(n^{(1)}, n^{(2)}| n^{(L)})$ vs. second-order joint contribution $C(n^{(1)}, n^{(2)}| n^{(L)})$.
         }
         \label{fig:corr_scatter_mnist_2nd_order}
     \end{subfigure}
     \hfill
     \begin{subfigure}[b]{\textwidth}
         \centering
         \includegraphics[width=\linewidth]{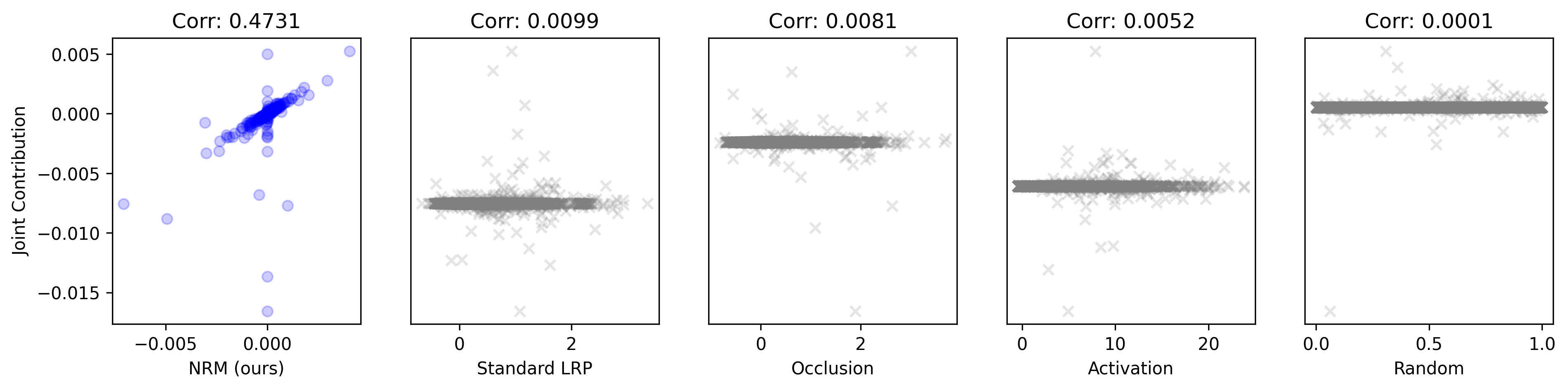}
         \caption{
         Third-order joint relevance $R(n^{(0)}, n^{(1)}, n^{(2)}| n^{(L)})$ vs. third-order joint contribution $C(n^{(0)}, n^{(1)}, n^{(2)}| n^{(L)})$.
         }
         \label{fig:corr_scatter_mnist_3rd_order}
     \end{subfigure}
    \caption{
        Correlation between joint relevance and joint contribution for MLP on a single test sample of MNIST.
    }
    \label{fig:corr_scatter_mnist}
\end{figure*}

\begin{figure*}[t]
    \centering
    \centering
     \begin{subfigure}[b]{\textwidth}
         \centering
         \includegraphics[width=\linewidth]{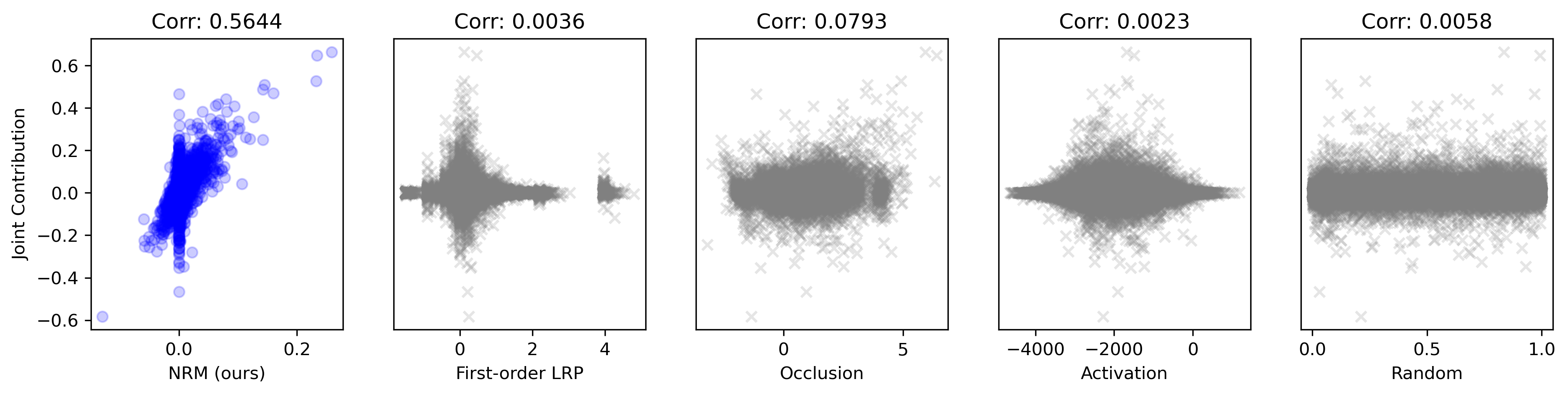}
         \caption{
         Channel-level.  Each point stands for a combination of one channel in layer 26 and one channel in layer 28.
         }
         \label{fig:corr_scatter_channel_2nd_order}
     \end{subfigure}
     \hfill
     \begin{subfigure}[b]{\textwidth}
         \centering
         \includegraphics[width=\linewidth]{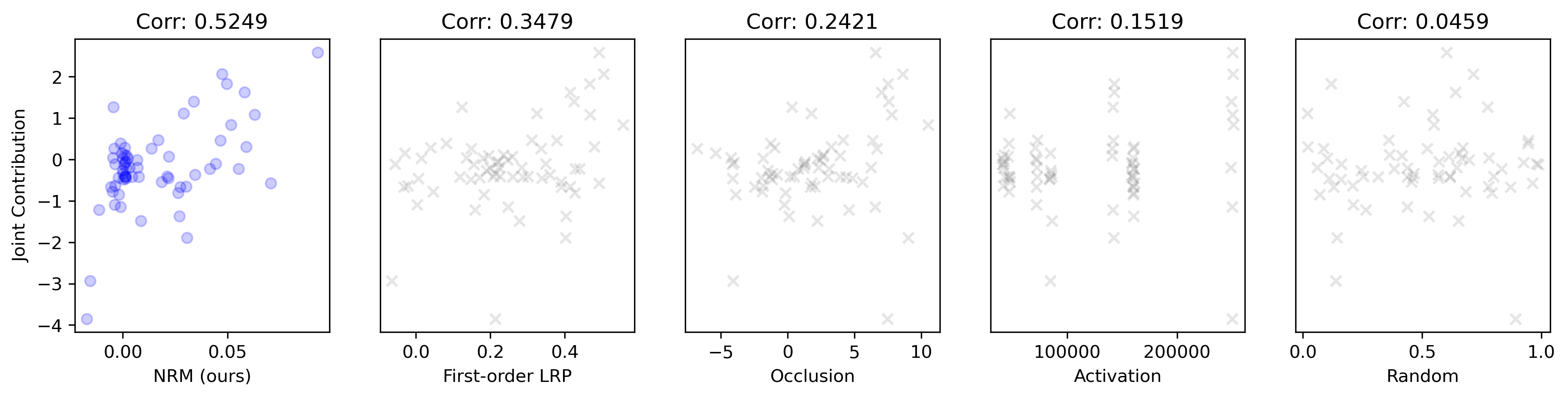}
         \caption{
         Cluster-level.  Each point stands for a combination of one cluster in intermediate embedding of block 1 and one in block 5.
         }
         \label{fig:corr_scatter_cluster_3rd_order}
     \end{subfigure}
    \caption{
        Correlation between joint relevance and joint contribution for VGG16 on a single test sample of ImageNet.      
    }
    \label{fig:corr_scatter_imagenet}
\end{figure*}

\section{Additional Results}
\label{sec:AdditionalResults}

Here we provide additional experimental results.

For the experiment with an MLP on MNIST in \Cref{sec:Experiment.MLP}, 
\Cref{fig:corr_scatter_mnist} visualizes the correlation for the second-order (top) and the third order (bottom) explanations, evaluated on the first input image of the test dataset. Strong correlation between the NRM joint relevance and the joint contribution is  observed.  In contrast, the correlation between the baseline explanation and the joint contribution is weak in general.

For the experiment with VGG16 on ImageNet in \Cref{sec:Experiment.VGG16}, 
\Cref{fig:corr_scatter_imagenet} similarly visualizes the correlation for the channel-level (top) and cluster-level (bottom) explanations, evaluated on the first input image of the test dataset. Similar tendency to the MNIST experiment is observed.

\section{NRM Definition for Recurrent Neural Networks (RNNs)}
\label{sec:DetailsLinear}

\begin{figure*}
    \centering
    \includegraphics[width=0.8\linewidth]{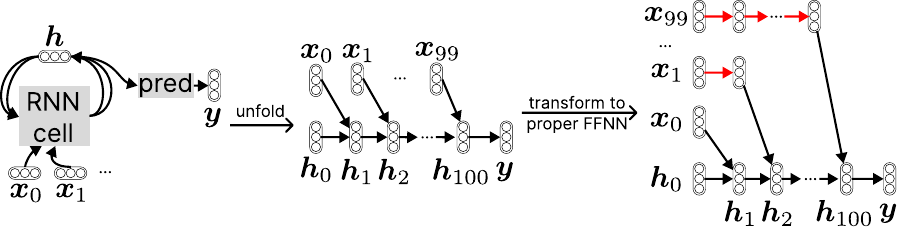}
    \caption{Conversion of RNN to proper FFNN. Left: The hidden state is interacted with a series of inputs, and the final hidden state is used as an input for the predictor, yielding the classification output. Middle: An unfolded representation of RNN, which is still not proper FFNN, because of the input neurons in intermediate layers. Right: A proper FFNN representation of RNN, obtained by duplicating  (red arrows) the input neurons up to the first layer.}
    \label{fig:rnn_proper_ffnn}
\end{figure*}

The RNN defined in Eqs.\eqref{eq:RNNForwardOne}--\eqref{eq:RNNForwardThree}
is not a proper FFNN (see \Cref{fig:rnn_proper_ffnn} left).
Therefore, we first convert it to a proper FFNN.
Following \ref{sec:GeneralNNAdaptation}, we first unfold the RNN cell (\Cref{fig:rnn_proper_ffnn} middle), and then duplicate input neurons so that all input neurons lie in the input layer,
giving a proper FNN representation (\Cref{fig:rnn_proper_ffnn} right).

The proper FFNN consists of $L = 101$ layers, each of which consists of the following set of neurons: 
\begin{align}
    \mcS^{(l)} &= 
    \begin{cases}
    \mcS^{(l)}_{\bfh_{l}}\cup 
    \left\{ \mcS^{(l)}_{\bfx_{l}}, \ldots, \mcS^{(l)}_{\bfx_{99}}\right\}, 
    & \mbox{ for } l = 0, \ldots, 99,\\
    \mcS^{(100)}_{\bfh_{100}},  
    & \mbox{ for } l = 100,
    \\
   \mcS^{(101)}_{\bfy},
   &  \mbox{ for } l = 101,
   \end{cases}
   \notag
\end{align}
as seen in \Cref{fig:rnn_proper_ffnn} (right).
For linear RNN and output cells,
\begin{align}
f^{\text{RNN}}(\bfh, \bfx) &= \bfW^h \bfh + \bfW^x \bfx, \notag \\
f^{\text{predictor}}(\bfh) &= \bfW^y\bfh,
\notag 
\end{align}
we define the unnormalized propagation matrix for $l=1, \ldots, 100$ as
\begin{align}
    \widetilde{\bfT}^{(l)}
    & = 
    \begin{pmatrix}
        \widetilde{\bfT}^{(l, \bfh_{l-1} \leftarrow \bfh_{l})} &   \bfzero \\
       \widetilde{\bfT}^{(l, \bfx_{l-1} \leftarrow \bfh_{l})} & \bfzero   \\
      \bfzero & \widetilde{\bfT}^{(l, \{\bfx_{l'}\}_{l' = l}^{99} \leftarrow \{\bfx_{l'}\}_{l' = l}^{99})} 
      \end{pmatrix},
        \label{eq:UnnormalizedPropagationMatrixRNN}
\end{align}
and for $l=101$ as
\begin{align}
    \widetilde{\bfT}^{(101)}
    & = 
        \widetilde{\bfT}^{(101, \bfh_{100} \leftarrow \bfy)} ,
        \label{eq:UnnormalizedPropagationMatrixRNN.Last}
\end{align}
where
\begin{align}
 \widetilde{T}_{n, n'}^{(l, \bfh_{l-1} \leftarrow \bfh_{l})}
 &= 
( h_{l-1})_{n} W^{h \uparrow}_{n, n'} + \varepsilon,
\notag\\
 \widetilde{T}_{n, n'}^{(l, \bfx_{l-1} \leftarrow \bfh_{l})} 
 &= 
 (x_{l-1})_{n} W^{x \uparrow}_{n, n'} + \varepsilon,
 \notag\\
 \widetilde{\bfT}^{(l, \{\bfx_{l'}\}_{l' = l}^{99} \leftarrow \{\bfx_{l'}\}_{l' = l}^{99})} 
 &= \bfI,
 \notag
\end{align}
and $\bfzero$ and $\bfI$ are the zero and identity matrices, respectively, with appropriate size.
Note that the unnormalized propagation matrix \eqref{eq:UnnormalizedPropagationMatrixRNN}
corresponds to Eq.~\eqref{eq:WholeUnnormalizedPropagationMatrix} with $\alpha = 1$ for all combinations of variables.

By normalizing Eqs.\eqref{eq:UnnormalizedPropagationMatrixRNN} and \eqref{eq:UnnormalizedPropagationMatrixRNN.Last},
we get the consecutive conditional relevance \eqref{eq:ConsecutiveConditionalRelevanceContruction} for each layer.  Defining the normalized output relevance \eqref{eq:OutputRelevanceContruction} by using the network output, we complete the relevance definition of all individual walks, and thus of all sets of walks.

\end{document}